\documentclass[11pt]{article}

\usepackage[final]{acl}

\usepackage{times}
\usepackage{latexsym}

\usepackage[T1]{fontenc}

\usepackage[utf8]{inputenc}

\usepackage{microtype}

\usepackage{inconsolata}

\usepackage{graphicx}
\usepackage{amsmath}
\usepackage{booktabs}
\usepackage{multirow}
\usepackage{amsmath}
\usepackage{amssymb}
\usepackage{amsthm}
\theoremstyle{definition}
\newtheorem{proposition}{Proposition}
\newtheorem{remark}{Remark}
\usepackage{tcolorbox}
\tcbuselibrary{breakable}
\usepackage{fvextra}
\RecustomVerbatimEnvironment{verbatim}{Verbatim}{breaklines=true}
\usepackage{colortbl}

\title{Creative Generation via Multi-Agent Debate: Does Debate Suppress Diversity?}

\author{
  \textbf{Tien Anh Nguyen}, \textbf{Khanh-Binh Nguyen}, \textbf{Van Dai Do}, \\
  \textbf{Svetha Venkatesh}, \textbf{Hung Le} \\
  Deakin Applied Artificial Intelligence Initiative, Deakin University, Australia \\
  \texttt{\{tien.nguyen,binh.nguyen,v.do,svetha.venkatesh,thai.le\}@deakin.edu.au}
}

\begin{document}
\maketitle
\begin{abstract}
Creative generation tasks, such as narrative writing and scientific ideation, demand both high-quality outputs and distinct responses across independent runs to maximize exploration. Multi-Agent Debate (MAD) has shown strong quality gains on factual and reasoning tasks, making it a natural candidate for creative generation. However, we find its convergence-driven design actively suppresses output diversity across independent runs, creating an inherent trade-off with creative tasks. We theoretically show that preserving diversity among agents within each debate session is a necessary condition for achieving diverse outputs across independent runs. Building on this finding, we propose \textbf{Creative-MAD}, which introduces two synergistic interventions to sustain agent divergence. Specifically, Cognitive Lens Assignment counters identity drift by anchoring each agent to a distinct and persistent cognitive mode, while Embedding-based Peer Selection counters majority pull by limiting each agent's context to its most semantically distant peers. Experiments across four creative benchmarks demonstrate that Creative-MAD significantly enhances both lexical and semantic diversity while maintaining MAD's output quality.
\end{abstract}

\section{Introduction}

Creative generation encompasses diverse tasks including narrative writing, scientific ideation, and argumentative composition~\cite{wu2026writingbench, ruan2026evaluating, huber-niklaus-2025-clear}. Unlike factual or reasoning tasks where a single correct answer exists, creative tasks are open-ended and fundamentally require output diversity. A system asked to generate responses to the same prompt across multiple runs should return meaningfully distinct outputs each time to explore different directions of the creative space~\cite{zhang2025noveltybench, ismayilzada-etal-2025-creative}. This diversity requirement makes creative generation challenging because it demands both high-quality outputs and meaningful variation across independent generations.

Multi-Agent Debate (MAD) has emerged as a promising paradigm for improving Large Language Model (LLM) output quality, where multiple model instances iteratively critique and refine each other's responses~\cite{du2024improving, liang2024encouraging, li2024improving}. The appeal is intuitive: exposure to competing perspectives surfaces errors, challenges weak reasoning, and drives convergence toward higher-quality outputs across debate rounds. Recent work has validated this appeal, demonstrating consistent quality gains over single-agent generation across factual question answering, arithmetic reasoning, and commonsense inference~\cite{choi2026debate}, making MAD a natural candidate for creative generation where output quality equally matters.

However, MAD's convergence-driven design introduces a fundamental tension in creative settings. Within each debate session, agents are progressively drawn toward a shared perspective as they read and respond to each other across rounds, causing diversity among agent responses to decay before a final output is selected. Since each independent run undergoes this same within-session diversity decay in isolation, the final outputs across runs end up clustering in the same narrow region, suppressing the diversity that creative generation requires. Yet such a quality-diversity trade-off has gone unnoticed in prior work, as MAD has been evaluated exclusively on output quality without examining whether independently generated outputs are meaningfully distinct from one another. We formalize the relationship between within-session diversity decay and output homogeneity across runs, showing that preserving diversity among agents within each debate session is a necessary condition for achieving diverse outputs across runs, and consequently, the only controllable lever for improving output diversity. 

Motivated by this finding, we propose \textbf{Creative-MAD}, a diversity-preserving multi-agent debate framework that addresses within-session convergence through two complementary mechanisms. The first targets \textit{identity drift}: when agents share identical prompt baselines, minor sampling variations fail to provide a permanent reasoning anchor, leaving multi-round interactions to homogenize individual agent behaviors. \emph{Cognitive Lens Assignment (CLA)} addresses this by equipping each agent with a distinct and persistent cognitive mode throughout debate, designed specifically for creative generation, such as analytical, emotional, or analogical reasoning. The second targets \textit{majority pull}: when each agent receives the full peer pool, dominant perspectives gradually erode minority viewpoints. \emph{Embedding-based Peer Selection (EPS)} addresses this by restricting each agent's context to its $k$ most semantically distant peers, replacing majority pull with cross-perspective stimulus. Our evaluation demonstrates that Creative-MAD significantly enhances cross-run diversity across both semantic and lexical dimensions without sacrificing the generation quality of standard MAD.

We summarize our contributions as follows:
\begin{itemize}

    \item We identify a fundamental quality-diversity trade-off in MAD for creative generation: while debate consistently improves output quality, the convergence dynamics among agents actively suppress output diversity across independent runs.

    \item We establish the theoretical basis of this trade-off, showing that preserving diversity among agents within each debate session is a necessary condition for achieving diverse outputs across multiple sessions.

    \item We propose Creative-MAD, combining Cognitive Lens Assignment and Embedding-based Peer Selection to counter identity drift and majority pull within debate sessions, maintaining output quality on par with standard MAD while substantially improving output diversity across a diverse suite of creative generation benchmarks.

\end{itemize}

\section{Related Work}
\subsection{Multi-Agent Debate}
MAD has demonstrated consistent quality gains on factual and reasoning tasks~\cite{du2024improving, choi2026debate}, where convergence toward a single correct answer is the desired outcome. To strengthen this convergence, prior work has explored diverse agent configurations through two main strategies: model-level heterogeneity, which leverages different LLM families~\cite{chen2024reconcile}, and role-level heterogeneity, which injects distinct personas or reasoning methods into agents of the same model~\cite{liang2024encouraging, wang2024unleashing}. However, these works treat diversity as a means to improve convergence toward a correct answer, not as an output requirement across independent runs. 

MAD's behavior on creative tasks, where output diversity is a feature rather than a flaw, has received limited attention \cite{becker2025mallm}. While \citet{hu2025debate} achieves argument diversity through per-run persona resampling, the intrinsic diversity decay during the debate process under a fixed prompt remains largely unaddressed. Consequently, investigating such decay as a primary failure mode in creative environments remains an unexplored gap, leaving a need for mechanisms that balance quality with the preservation of diverse creative trajectories.

\subsection{LLM Creative Generation and Evaluation}
Recent benchmarks have expanded creative evaluation into specialized domains, including physical problem-solving \cite{tian2024macgyver}, scientific ideation \cite{ruan2026evaluating}, and narrative writing \cite{li2025from}. LLM-as-a-Judge (LLMaaJ) frameworks \cite{wu2026writingbench}, pairwise Elo-based systems \cite{paech2023eq}, and Win Rate \cite{li2025from} have become widely adopted for open-ended evaluation. A critical challenge in creative generation is mode collapse, where LLMs produce repetitive outputs for the same prompt \cite{zhang2025noveltybench}. Existing work addresses this via diversity-aware training \cite{chung2025modifying} or quality-weighted diversity metrics \cite{shypula2025evaluating}, but the multi-agent setting remains underexplored.

\section{Problem Formulation}
\label{sec:problem}

\subsection{MAD Formulation}

We consider a multi-agent debate system comprising $N$ agents operating over $R$ rounds to produce a single response for a given input query $x$. Formally, let $\mathcal{A} = \{a_1, a_2, \ldots, a_N\}$ denote the set of agents, where each agent is an instance of the same language model $\mathcal{M}$. We focus on this single-model setting to cleanly isolate the effect of debate dynamics from inherent capability differences across model families. We follow the common simultaneous-talk protocol~\cite{chan2024chateval, choi2026debate}, where all agents update concurrently based on the previous round's responses.

At round $t = 0$, each agent $a_i$ independently generates an initial response $y_i^{(0)} = \mathcal{M}(x)$. At each subsequent round $t \in \{1, \ldots, R\}$, each agent refines its response by conditioning on both the original query and a subset of peer responses from the previous round:
\begin{equation}
    y_i^{(t)} = \mathcal{M}\left(x,\ \mathcal{P}_i^{(t-1)}\right)
\end{equation}
where $\mathcal{P}_i^{(t-1)} \subseteq \{y_j^{(t-1)} : j \neq i\}$ denotes the set of peer responses visible to agent $i$ at round $t$. In standard MAD, $\mathcal{P}_i^{(t-1)}$ contains all peer responses. After $R$ debate rounds, a consensus mechanism selects a single final answer $\hat{y}$ from the final agent responses $\{y_i^{(R)}\}_{i=1}^N$. To evaluate robustness and diversity across generations, we run $G$ independent sessions for each query, yielding a set of final outputs $\{\hat{y}_g\}_{g=1}^G$.

\subsection{Creative Task Setting}

Given an input query $x$, the goal is to generate a response $y$ evaluated along two complementary dimensions: \textit{instance-level quality} (coherence, creativity, and task alignment) and \textit{set-level diversity} (the degree to which multiple generated outputs explore distinct regions of the output space).

Formally, given $G$ independent sessions on the same fixed query $x$ producing outputs $\{\hat{y}_g\}_{g=1}^G$, the objective is to jointly maximize instance-level quality $\frac{1}{G}\sum_{g=1}^G\text{Quality}(\hat{y}_g)$ and set-level diversity $\text{Diversity}(\{\hat{y}_g\}_{g=1}^G)$ across all sessions, where the latter requires outputs to be meaningfully distinct from one another, a property we term \textit{inter-session diversity}.

\paragraph{Instance-level quality.} We measure quality via absolute scores using LLMaaJ~\cite{ruan2026evaluating}, and average pairwise win rate~\cite{li2025from} across all method pairs.

\paragraph{Set-level diversity.} We measure diversity along two axes. For semantic diversity, we adopt the Vendi Score~\citep{friedman2023the}, defined as:
\begin{equation}
    \text{VS} = \exp\left(-\operatorname{tr}
    \left(\bar{K} \log \bar{K}\right)\right), \quad 
    \bar{K} = \frac{K}{G}
\end{equation}
where $\operatorname{tr}(\cdot)$ denotes the matrix trace, $K \in \mathbb{R}^{G \times G}$ is the cosine similarity kernel matrix with entries $K_{gg'} = {e_g \cdot e_{g'}} / {\|e_g\|\|e_{g'}\|}$ for $g, g' \in \{1, \ldots, G\}$, and $e_g$ denotes the embedding of final output $\hat{y}_g$. The score ranges from $1$ for identical outputs to $G$ for fully distinct outputs. For lexical diversity, we use Div-BLEU~\citep{ruan2025g2}, calculated as $1 - \text{Self-BLEU}$~\citep{zhu2018texygen}.

\section{Why MAD Suppresses Diversity in Creative Generation}
\label{sec:theorical}

\begin{figure*}[t]
    \centering
    \includegraphics[width=\textwidth]{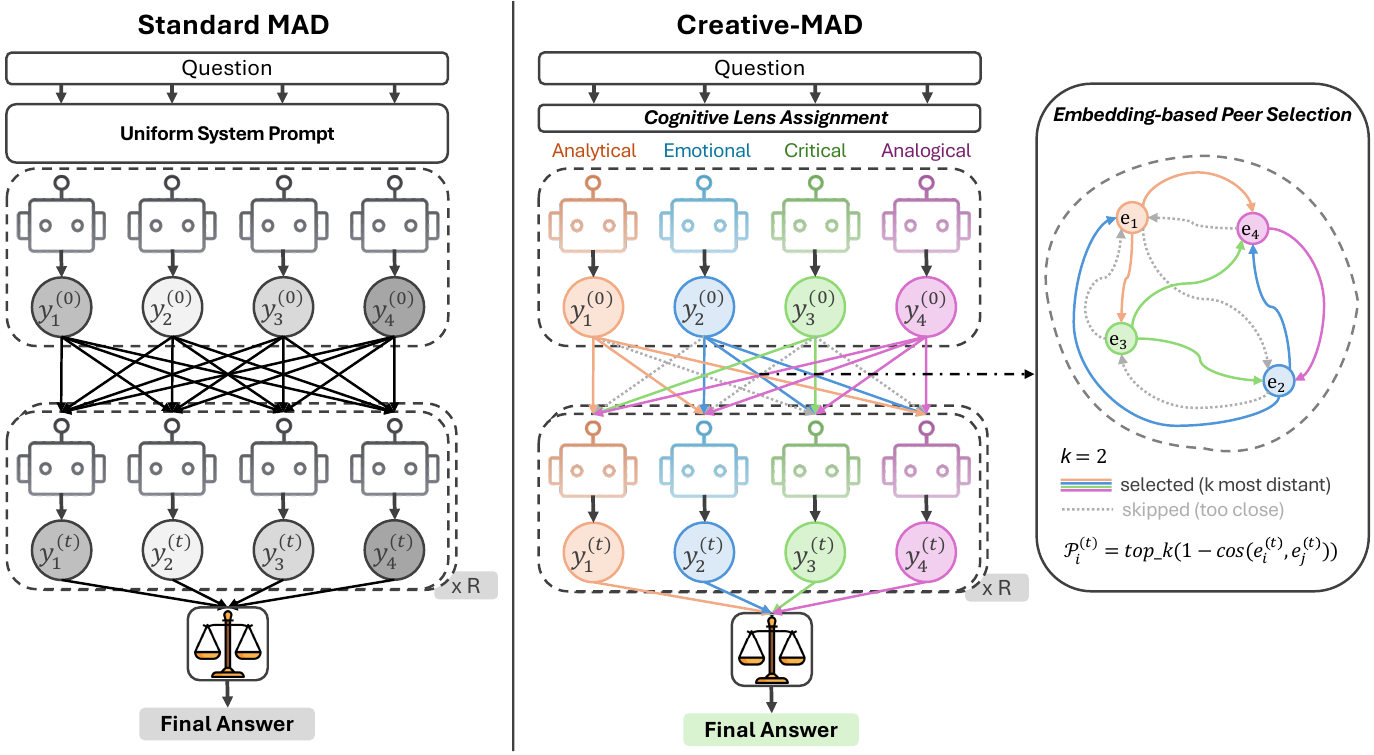}
    \caption{Comparison of Standard MAD and Creative-MAD. Standard MAD's uniform prompts and full peer connectivity drive convergence across debate rounds. Creative-MAD addresses this via Cognitive Lens Assignment (CLA), anchoring each agent to a distinct cognitive mode, and Embedding-based Peer Selection (EPS), restricting each agent to its $k$ most semantically distant peers.}
    
    \label{fig:creative_mad}
\end{figure*}

While MAD's convergence dynamics yield consistent quality gains over single-agent baselines, they introduce a fundamental quality-diversity trade-off in creative settings: the same inter-agent communication that drives quality improvement actively suppresses output diversity across independent runs.

\paragraph{Intra-session Diversity Decay.}
We define \textit{intra-session diversity} as the semantic spread among the $N$ agent responses $\{y_i^{(t)}\}_{i=1}^{N}$ at debate round $t$, measured by the semantic diversity score over agent outputs within a single session. Standard MAD drives monotonic decay of intra-session diversity through two distinct mechanisms. First, identity drift: as agents are initialized with identical prompts and differ only through temperature sampling, they lack a stable reasoning anchor, causing them to drift toward group consensus as they synthesize peer responses across rounds. Second, majority pull: full peer connectivity exposes each agent to the majority signal, causing dominant perspectives to gradually override minority viewpoints and homogenize agent responses. As intra-session diversity decays, the pool from which the consensus mechanism selects the final output $\hat{y}_g$ becomes increasingly concentrated, causing the final outputs $\{\hat{y}_g\}_{g=1}^G$ across independent sessions to cluster in the same narrow region of the output space.

\paragraph{Theoretical Analysis.}
We formalize the relationship between intra-session diversity decay and inter-session output homogeneity as follows.

\begin{proposition}[Intra-to-Inter Diversity]
\label{prop:intra_inter}
Let $K \in \mathbb{R}^{G \times G}$ be the cosine similarity kernel matrix over final outputs $\{\hat{y}_g\}_{g=1}^{G}$ and $\mathcal{D}_{\mathrm{inter}} = \exp(H(K))$ the inter-session diversity score where
\begin{equation}
    H(K) = -\operatorname{tr}(\bar{K} \log \bar{K})
\end{equation}
is the kernel entropy~\cite{friedman2023the}. A decrease in intra-session diversity necessarily induces a decrease in $\mathcal{D}_{\mathrm{inter}}$, and preserving intra-session diversity is a necessary condition for maintaining high inter-session diversity under the fixed-prompt setting. 
\end{proposition}
\begin{proof}
    See Appendix~\ref{appendix:proof}.
\end{proof}

\section{Creative-MAD: Diversity-Preserving Multi-Agent 
Debate for Creative Generation}
\label{sec:method}

Standard MAD suppresses intra-session diversity through two mechanisms: identity drift and majority pull. We propose \textbf{Creative-MAD}, with two complementary interventions targeting each mechanism: CLA addresses identity drift by anchoring each agent to a distinct processing mode throughout debate, while EPS addresses majority pull by filtering the peer signal to amplify cross-perspective stimulus. An overview of the full architecture is illustrated in Figure~\ref{fig:creative_mad}.

\subsection{Cognitive Lens Assignment (CLA)}

Creative generation is inherently multi-dimensional: a strong response to the same prompt may emerge from an emotionally resonant narrative, a logically structured argument, a counterintuitive reframing, or an unexpected cross-domain analogy, each constituting a legitimate and distinct creative direction. This multiplicity underlies the diversity requirement central to creative generation, yet standard MAD collapses it: since all agents share identical system prompts and differ only through temperature sampling, there is no structural mechanism to maintain distinctive generative orientations across debate rounds.

We address this by assigning each agent a distinct cognitive lens, a persistent system prompt that explicitly defines a specific cognitive processing mode. Unlike domain personas that specify ``what'' an agent knows (e.g., scientist, lawyer), cognitive lenses specify ``how'' an agent processes information, covering complementary modes such as logical analysis, emotional reasoning, critical examination, analogical thinking, and practical evaluation. Cognitive lenses are assigned deterministically across agents, ensuring full cognitive coverage within each session. The lens is explicitly reinforced in the debate prompt at every round, preventing the identity drift that occurs when agents are exposed to peer responses during synthesis. Full lens descriptions are provided in Appendix~\ref{appendix:cognitive_lenses}.

\begin{table*}[t]
\centering
\small
\caption{Instance-level quality results across four creative benchmarks. We report rubric-based absolute scores (mean $\pm$ std) and average pairwise win rate (\%) across all method pairs. Higher is better. \textbf{Bold} indicates the best result(s); multiple entries may be bolded if not statistically significantly different (paired $t$-test, $p < 0.05$).}
\label{tab:quality}
\resizebox{\textwidth}{!}{%
\begin{tabular}{l cc cc cc cc}
\toprule
\multirow{2}{*}{\textbf{Method}} 
& \multicolumn{2}{c}{\textbf{LiveIdea}} 
& \multicolumn{2}{c}{\textbf{AAE}} 
& \multicolumn{2}{c}{\textbf{MacGyver}} 
& \multicolumn{2}{c}{\textbf{AH v2.0}} \\
\cmidrule(lr){2-3} \cmidrule(lr){4-5} \cmidrule(lr){6-7} \cmidrule(lr){8-9}
& Score & Win\% & Score & Win\% & Score & Win\% & Score & Win\% \\
\midrule
\multicolumn{9}{l}{\textit{Qwen3-8B}} \\
\cmidrule(r){1-9}
Direct       & 66.5 $\pm$ 0.1 & 7.6  & 73.7 $\pm$ 0.1 & 4.0  & 64.2 $\pm$ 0.8 & 6.1  & 74.9 $\pm$ 0.3 & 10.6 \\
Self-Refine  & 69.7 $\pm$ 0.3 & 29.3 & 74.7 $\pm$ 0.1 & 8.9  & 68.2 $\pm$ 0.1 & 18.5 & 76.6 $\pm$ 0.2 & 17.3 \\
Voting       & 67.2 $\pm$ 0.1 & 11.8 & 74.2 $\pm$ 0.4 & 5.8  & 66.9 $\pm$ 0.3 & 12.7 & 76.5 $\pm$ 0.3 & 19.0 \\
Homo MAD     & \textbf{74.0} $\pm$ 0.1 & \textbf{52.7} & 80.6 $\pm$ 0.5 & 57.1 & \textbf{76.3} $\pm$ 0.3 & 63.4 & 80.2 $\pm$ 0.3 & 44.3 \\
Hetero MAD   & 72.2 $\pm$ 0.2 & 51.8 & \textbf{82.3} $\pm$ 0.3 & 63.7 & \textbf{76.6} $\pm$ 0.1 & \textbf{65.3} & \textbf{81.4} $\pm$ 0.3 & \textbf{51.2} \\
\rowcolor{blue!10} \textbf{Creative-MAD} & 71.5 $\pm$ 0.1 & 49.4 & \textbf{82.8} $\pm$ 0.4 & \textbf{65.1} & \textbf{76.4} $\pm$ 0.3 & \textbf{65.6} & 80.4 $\pm$ 0.9 & 45.7 \\
\midrule
\multicolumn{9}{l}{\textit{Gemma-3-12B-Instruct}} \\
\cmidrule(r){1-9}
Direct            & 75.3 $\pm$ 0.2 & 8.9  & 82.1 $\pm$ 0.1 & 5.3  & 76.9 $\pm$ 0.3 & 10.4 & 78.0 $\pm$ 0.3 & 11.6 \\
Self-Refine       & 81.8 $\pm$ 0.4 & 26.7 & 83.9 $\pm$ 0.3 & 8.9  & 78.9 $\pm$ 0.4 & 23.8 & 80.2 $\pm$ 0.1 & 23.4 \\
Voting            & 76.3 $\pm$ 0.7 & 15.7 & 82.4 $\pm$ 0.2 & 9.6  & 77.4 $\pm$ 0.2 & 19.9 & 80.9 $\pm$ 0.2 & 23.9 \\
Homo MAD          & \textbf{82.5} $\pm$ 0.3 & \textbf{55.6} & 84.9 $\pm$ 0.5 & 62.1 & \textbf{81.1} $\pm$ 0.5 & \textbf{66.1} & \textbf{84.3} $\pm$ 0.3 & \textbf{68.2} \\
Hetero MAD        & 81.9 $\pm$ 0.2 & 51.1 & \textbf{85.0} $\pm$ 0.3 & \textbf{65.8} & \textbf{81.0} $\pm$ 0.4 & 65.4 & \textbf{84.1} $\pm$ 0.3 & 67.5 \\
\rowcolor{blue!10} \textbf{Creative-MAD}      & \textbf{82.2} $\pm$ 0.4 & \textbf{55.4} & \textbf{85.1} $\pm$ 0.4 & \textbf{65.7} & 80.2 $\pm$ 0.2 & 62.9 & \textbf{84.2} $\pm$ 1.2 & \textbf{68.0} \\
\bottomrule
\end{tabular}%
}
\end{table*}

\subsection{Embedding-based Peer Selection (EPS)}

Even with distinct cognitive lenses, agents remain susceptible to majority-pull if each agent receives the full set of peer responses during debate. When all agents observe the same peer pool, the most frequently represented perspectives dominate synthesis, gradually eroding cognitive distinctiveness regardless of lens assignment.

We address this through EPS, which filters the peer signal each agent receives. At each debate round, all agent responses are embedded into a shared semantic space via an embedding encoder $\mathcal{E}$, where $e_i^{(t)} = \mathcal{E}(y_i^{(t)})$. Each agent then receives only its $k$ most semantically distant peers as debate input, rather than the full peer pool. Formally, at debate round $t \geq 1$, for agent $i$ with current response $y_i^{(t)}$, the peer set is defined as:
\begin{equation}
    \mathcal{P}_i^{(t)} = \underset{j \neq i}{\mathrm{top}\text{-}k} \left(1 - \cos\bigl(e_i^{(t)},\, e_j^{(t)}\bigr)\right)
\end{equation}
where $\cos(\cdot,\cdot)$ denotes cosine similarity. By selecting the most semantically distant peers, EPS ensures that the creative stimulus each agent receives is maximally different from its own current direction, amplifying exploration over exploitation.

\section{Experiments}
\label{sec:experiments}

\begin{table*}[t]
\centering
\small
\caption{Set-level diversity results across benchmarks. We report semantic diversity score (Sem., mean $\pm$ std) and lexical diversity score (Lex., mean $\pm$ std, scaled by $\times 100$). Higher is better. \textbf{Bold} indicates the best result(s); multiple entries may be bolded if not statistically significantly different (paired $t$-test, $p < 0.05$).}
\label{tab:diversity}
\resizebox{\textwidth}{!}{%
\begin{tabular}{l cc cc cc cc}
\toprule
\multirow{2}{*}{\textbf{Method}} 
& \multicolumn{2}{c}{\textbf{LiveIdea}} 
& \multicolumn{2}{c}{\textbf{AAE}} 
& \multicolumn{2}{c}{\textbf{MacGyver}} 
& \multicolumn{2}{c}{\textbf{AH v2.0}} \\
\cmidrule(lr){2-3} \cmidrule(lr){4-5} \cmidrule(lr){6-7} \cmidrule(lr){8-9}
& Sem. & Lex. & Sem. & Lex. & Sem. & Lex. & Sem. & Lex. \\
\midrule
\multicolumn{9}{l}{\textit{Qwen3-8B}} \\
\cmidrule(r){1-9}
Direct       & 2.54 $\pm$ 0.05 & 68.96 $\pm$ 1.1 & 1.52 $\pm$ 0.04 & 52.56 $\pm$ 0.9 & 1.52 $\pm$ 0.05 & 55.09 $\pm$ 1.2 & 2.24 $\pm$ 0.04 & 71.98 $\pm$ 0.8 \\
Self-Refine  & 2.59 $\pm$ 0.04 & 75.70 $\pm$ 1.4 & 1.58 $\pm$ 0.05 & 56.86 $\pm$ 1.1 & 1.53 $\pm$ 0.04 & 57.85 $\pm$ 1.3 & 2.25 $\pm$ 0.05 & 72.00 $\pm$ 0.9 \\
Voting       & 2.58 $\pm$ 0.06 & 69.68 $\pm$ 0.9 & 1.54 $\pm$ 0.04 & 53.61 $\pm$ 1.2 & 1.52 $\pm$ 0.05 & 55.75 $\pm$ 0.9 & 2.21 $\pm$ 0.06 & 71.85 $\pm$ 1.4 \\
Homo MAD     & 2.47 $\pm$ 0.04 & 66.62 $\pm$ 1.3 & 1.47 $\pm$ 0.05 & 52.31 $\pm$ 0.8 & 1.48 $\pm$ 0.04 & 55.63 $\pm$ 1.2 & 2.21 $\pm$ 0.04 & 68.51 $\pm$ 1.1 \\
Hetero MAD   & 2.58 $\pm$ 0.05 & 70.87 $\pm$ 1.0 & 1.67 $\pm$ 0.06 & 58.99 $\pm$ 1.5 & 1.53 $\pm$ 0.05 & 57.19 $\pm$ 0.8 & 2.27 $\pm$ 0.05 & 72.72 $\pm$ 1.2 \\
\rowcolor{blue!10} \textbf{Creative-MAD} & \textbf{3.04} $\pm$ 0.06 & \textbf{83.20} $\pm$ 1.2 & \textbf{2.24} $\pm$ 0.04 & \textbf{69.24} $\pm$ 1.3 & \textbf{1.67} $\pm$ 0.04 & \textbf{66.81} $\pm$ 1.0 & \textbf{2.68} $\pm$ 0.05 & \textbf{82.05} $\pm$ 1.2 \\
\midrule
\multicolumn{9}{l}{\textit{Gemma-3-12B-Instruct}} \\
\cmidrule(r){1-9}
Direct       & 2.69 $\pm$ 0.04 & 69.12 $\pm$ 0.9 & 1.71 $\pm$ 0.05 & 50.07 $\pm$ 1.2 & 1.67 $\pm$ 0.05 & 47.70 $\pm$ 1.4 & 2.64 $\pm$ 0.04 & 64.73 $\pm$ 0.8 \\
Self-Refine  & 2.72 $\pm$ 0.05 & 73.05 $\pm$ 1.4 & 1.77 $\pm$ 0.04 & 53.57 $\pm$ 0.9 & 1.67 $\pm$ 0.06 & 52.12 $\pm$ 1.1 & 2.57 $\pm$ 0.05 & 63.89 $\pm$ 1.3 \\
Voting       & 2.66 $\pm$ 0.04 & 69.32 $\pm$ 1.1 & 1.72 $\pm$ 0.05 & 50.28 $\pm$ 1.3 & 1.66 $\pm$ 0.04 & 48.60 $\pm$ 0.9 & 2.61 $\pm$ 0.04 & 64.86 $\pm$ 1.0 \\
Homo MAD     & 2.44 $\pm$ 0.05 & 63.28 $\pm$ 1.4 & 1.68 $\pm$ 0.04 & 49.46 $\pm$ 1.1 & 1.60 $\pm$ 0.05 & 48.34 $\pm$ 1.2 & 2.58 $\pm$ 0.06 & 63.07 $\pm$ 1.5 \\
Hetero MAD   & 2.70 $\pm$ 0.06 & 70.58 $\pm$ 0.8 & 1.75 $\pm$ 0.06 & 52.46 $\pm$ 1.2 & 1.68 $\pm$ 0.04 & 52.50 $\pm$ 1.3 & 2.65 $\pm$ 0.05 & 64.49 $\pm$ 1.0 \\
\rowcolor{blue!10} \textbf{Creative-MAD} & \textbf{3.51} $\pm$ 0.05 & \textbf{81.97} $\pm$ 1.2 & \textbf{1.92} $\pm$ 0.05 & \textbf{61.05} $\pm$ 1.4 & \textbf{1.82} $\pm$ 0.06 & \textbf{59.73} $\pm$ 1.0 & \textbf{2.98} $\pm$ 0.04 & \textbf{75.12} $\pm$ 1.1 \\
\bottomrule
\end{tabular}%
}
\end{table*}

\subsection{Experimental Setup}

\paragraph{Benchmarks.} We evaluate across four benchmarks spanning diverse dimensions of creative generation: \textit{LiveIdeaBench} (scientific ideation) \cite{ruan2026evaluating}, \textit{Argument Annotated Essays (AAE)} (argumentative writing) \cite{huber-niklaus-2025-clear}, \textit{MacGyver} (creative problem-solving) \cite{tian2024macgyver}, and \textit{Arena Hard v2.0 (AH v2.0)} (general creative writing) \cite{paech2023eq}. Appendix~\ref{appendix:dataset_details} provides dataset details.

\paragraph{Baselines.} We compare against the following methods: (1) \textit{Direct}, a single-agent baseline; (2) \textit{Self-Refine}~\cite{madaan2023self}, which iteratively refines a response using self-feedback; (3) \textit{Voting}, which aggregates initial responses from $N$ agents without debate; (4) \textit{Homogeneous MAD (Homo MAD)}~\cite{du2024improving}, the standard decentralized multi-agent debate; (5) \textit{Heterogeneous MAD (Hetero MAD)}~\cite{choi2026debate}, which assigns distinct domain personas to agents before debate; and (6) \textit{Creative-MAD}\footnote{Code available at \url{https://github.com/DA2I2-SLM/Creative-MAD}.} (ours), which combines CLA and EPS to preserve intra-session diversity. All methods use $N = 5$, $R = 2$ rounds, and temperature $= 1.0$. Creative-MAD additionally uses $k=2$ for EPS. Since creative tasks produce open-ended outputs that cannot be compared by exact match or majority voting, all voting and debate-based methods employ a judge instantiated with the same underlying model as the debating agents to select the best candidate at the consensus stage.

\paragraph{Evaluation.} For instance-level quality, we employ an \textit{LLMaaJ} using Qwen3.5-397B-A17B~\cite{qwen3.5}, reporting rubric-based absolute scores and average pairwise win rate computed across all runs per query~\cite{li2025from}. For set-level diversity, we measure semantic diversity via \textit{Vendi Score}~\cite{friedman2023the} computed over \texttt{all-MiniLM-L6-v2} representations~\cite{wang2020minilm} and lexical diversity via \textit{Div-BLEU}~\cite{ruan2025g2}. Each method is independently run 15 times per query. For instance-level quality, we report the mean and standard deviation across all 15 runs. For set-level diversity, we set $G=5$, partition the 15 runs into 3 groups, and report the mean and standard deviation of diversity scores across groups.

\paragraph{Models.} Experiments use two open-source models: Qwen3-8B~\cite{yang2025qwen3} and Gemma-3-12B-Instruct~\cite{gemma3team2025}. 

\subsection{Main Results}

\paragraph{Instance-level Quality Results.}
Table~\ref{tab:quality} presents instance-level quality results across all methods and benchmarks. Creative-MAD maintains quality on par with Homo MAD and Hetero MAD, ranking among the top-performing methods in 5 out of 8 comparisons on both absolute score and win rate across both models. The average score gap to the best-performing method remains within 1.2\% on Qwen3-8B and 0.4\% on Gemma-3-12B-Instruct, and the average win rate gap remains within 2.2\% on Qwen3-8B and 0.9\% on Gemma-3-12B-Instruct, demonstrating that the diversity-preserving mechanisms do not interfere with the quality gains that debate provides. Furthermore, all MAD variants consistently outperform single-agent baselines on both metrics, including Voting, which uses the same number of agents but without debate, confirming that inter-agent communication contributes quality gains beyond simple response aggregation.

\paragraph{Set-level Diversity Results.}
Table~\ref{tab:diversity} presents set-level diversity results across methods and benchmarks. Creative-MAD consistently achieves the highest semantic and lexical diversity, outperforming Homo MAD by 26.2\% (semantic) and 24.0\% (lexical) on Qwen3-8B, and by 23.3\% and 24.0\% on Gemma-3-12B-Instruct. Conversely, Homo MAD yields the lowest semantic diversity, falling below single-agent baselines. While Hetero MAD recovers some diversity through domain-role differentiation, its semantic diversity remains significantly below Creative-MAD by 19.7\% on Qwen3-8B and 16.5\% on Gemma-3-12B-Instruct, which indicates that persona assignment alone fails to counteract the convergence pressure of debate.

Notably, Voting achieves diversity comparable to Direct despite applying the same judge-based consensus mechanism to $N$ independent responses. This confirms that the consensus mechanism does not drive diversity collapse, thereby isolating debate dynamics as the primary driver of inter-session diversity suppression.

\subsection{Human Evaluation}
\label{sec:human_eval}

To validate that our automated evaluation reflects human judgment, we conduct two complementary human studies; full protocols and results are in Appendix~\ref{appendix:human_agreement}.

\paragraph{Judge reliability.} Since our main results rely on Qwen3.5-397B-A17B as judge, we first verify that its quality assessments track human preference. Three annotators pairwise-compare 120 method outputs (Direct/Self-Refine vs.\ Homo MAD) across all four benchmarks. Human annotators reach $79.0\%$ pairwise agreement among themselves, and our LLMaaJ agrees with the human majority vote in $81.0\%$ of cases, on par with human-level agreement.

\paragraph{Creative-MAD vs.\ Homo MAD.} To confirm that Creative-MAD's diversity gains and quality parity also hold under direct human judgment, we run a pilot study comparing Creative-MAD against Homo MAD on 40 queries. For quality, human judgments are split closely between the two methods, with ties as the most common outcome, consistent with our claim that Creative-MAD preserves output quality. For diversity, human annotators judge Creative-MAD's output set as more diverse in $35/40$ queries, matching the Vendi Score's preference in $88\%$ of cases. Together, these studies support that both our quality and diversity findings align with human judgment.

\section{Model Analysis}
\label{sec:analysis}

\begin{figure}[t]
    \centering
    \includegraphics[width=\columnwidth]{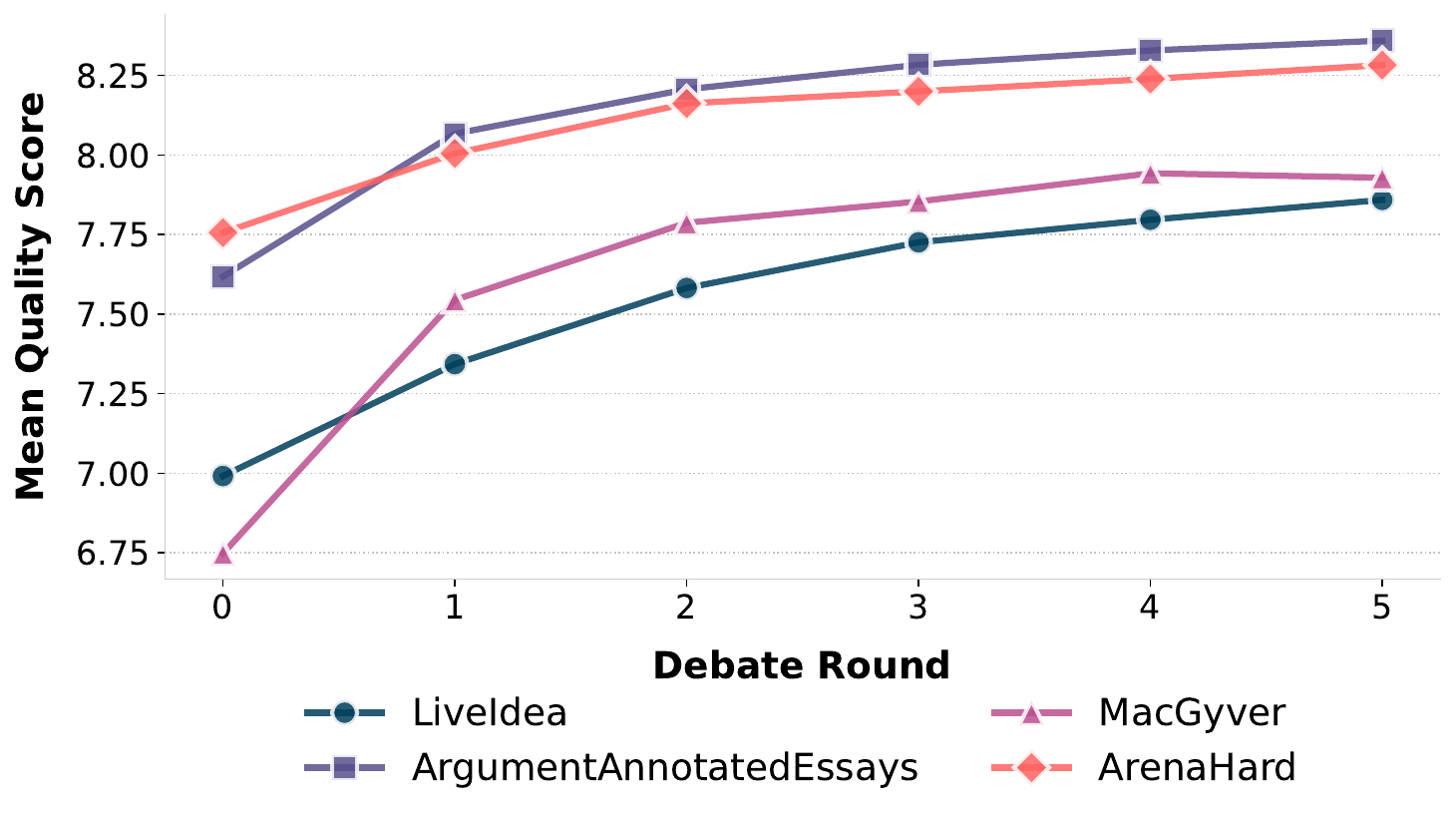}
    \caption{Mean quality score of Creative-MAD across 
    debate rounds on Qwen3-8B.}
    
    \label{fig:quality_rounds}
\end{figure}

\subsection{Quality Improvement Across Debate Rounds}
Figure~\ref{fig:quality_rounds} shows that Creative-MAD's quality improves progressively across debate rounds, confirming that debate drives quality gains beyond mere response aggregation. This behavior stands in contrast to the martingale property reported by \citet{choi2026debate} on factual tasks, where agent beliefs remain stationary in expectation across debate rounds. We attribute this discrepancy to a fundamental difference in task structure: unlike factual tasks, creative tasks involve open-ended response spaces and continuous quality signals, violating the categorical and Bayesian conjugacy assumptions that the martingale result relies on (see Appendix~\ref{appendix:martingale} for details).

\begin{figure}[t]
    \centering
    \includegraphics[width=\columnwidth]{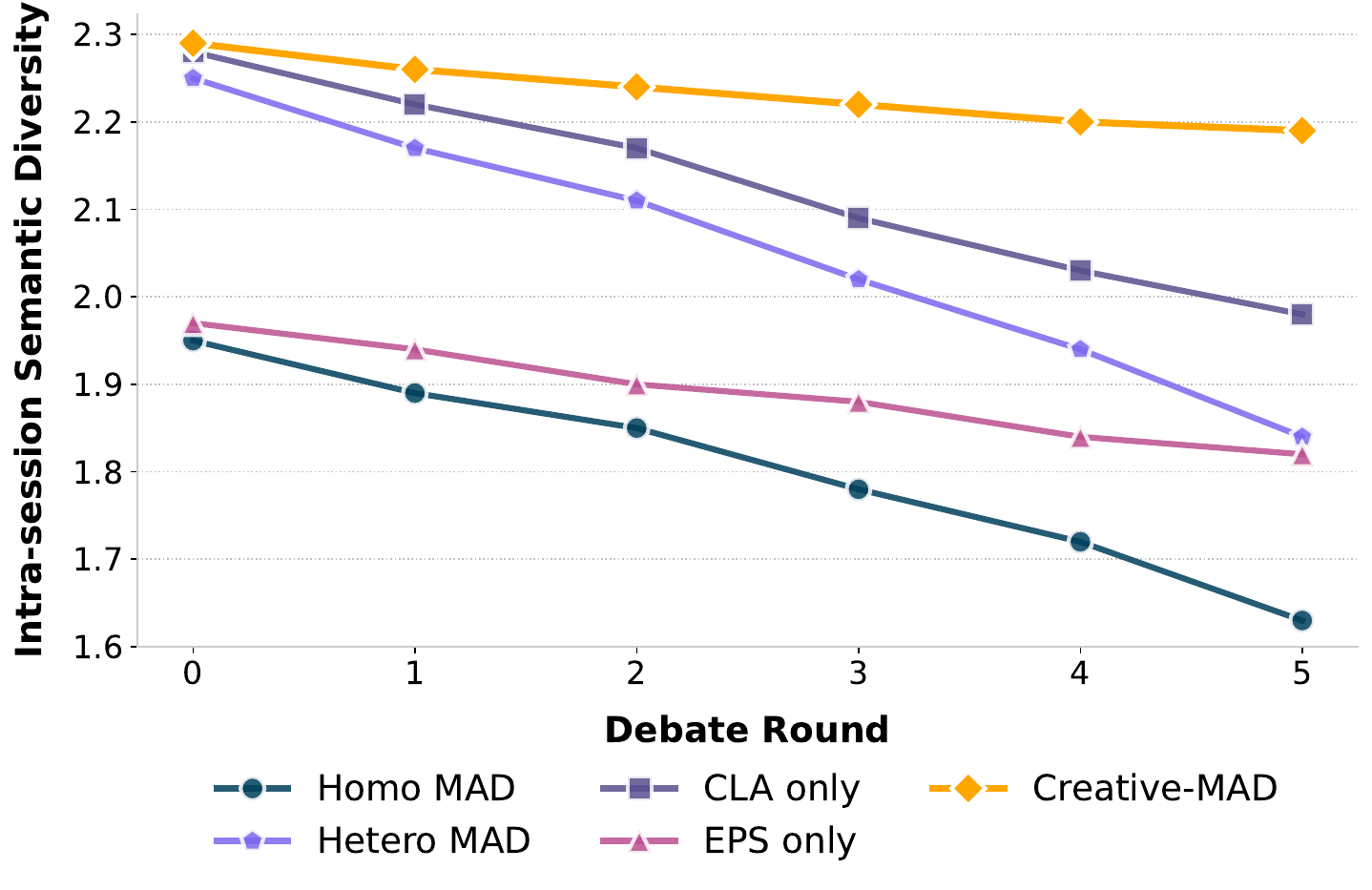}
    \caption{Intra-session semantic diversity across debate rounds for each method variant on Qwen3-8B, averaged across four benchmarks.}
    
    \label{fig:intra_comparison}
\end{figure}

\subsection{Intra-session Diversity}

To analyze the contribution of each individual component, we evaluate two variants: CLA only, which applies CLA without EPS, and EPS only, which applies EPS without CLA.

Figure~\ref{fig:intra_comparison} illustrates intra-session semantic diversity across debate rounds for each method variant on Qwen3-8B. At $t=0$, Creative-MAD, CLA only, and Hetero MAD begin at comparable diversity levels, while EPS only starts lower as it does not affect agent initialization. Across subsequent rounds, Hetero MAD decays at a rate comparable to Homo MAD despite its higher starting point, confirming that domain personas fail to sustain cognitive distinctiveness under multi-turn peer pressure. CLA only decays more slowly, while EPS only, despite starting lower, exhibits the slowest decay rate, confirming that filtering peer signals is more effective at counteracting majority pull. Creative-MAD combines both effects, achieving the strongest overall diversity preservation, providing a mechanistic explanation for the inter-session diversity gains in Table~\ref{tab:diversity}.

\subsection{Inter-session Diversity}

\begin{figure}[t]
    \centering
    \includegraphics[width=\columnwidth]{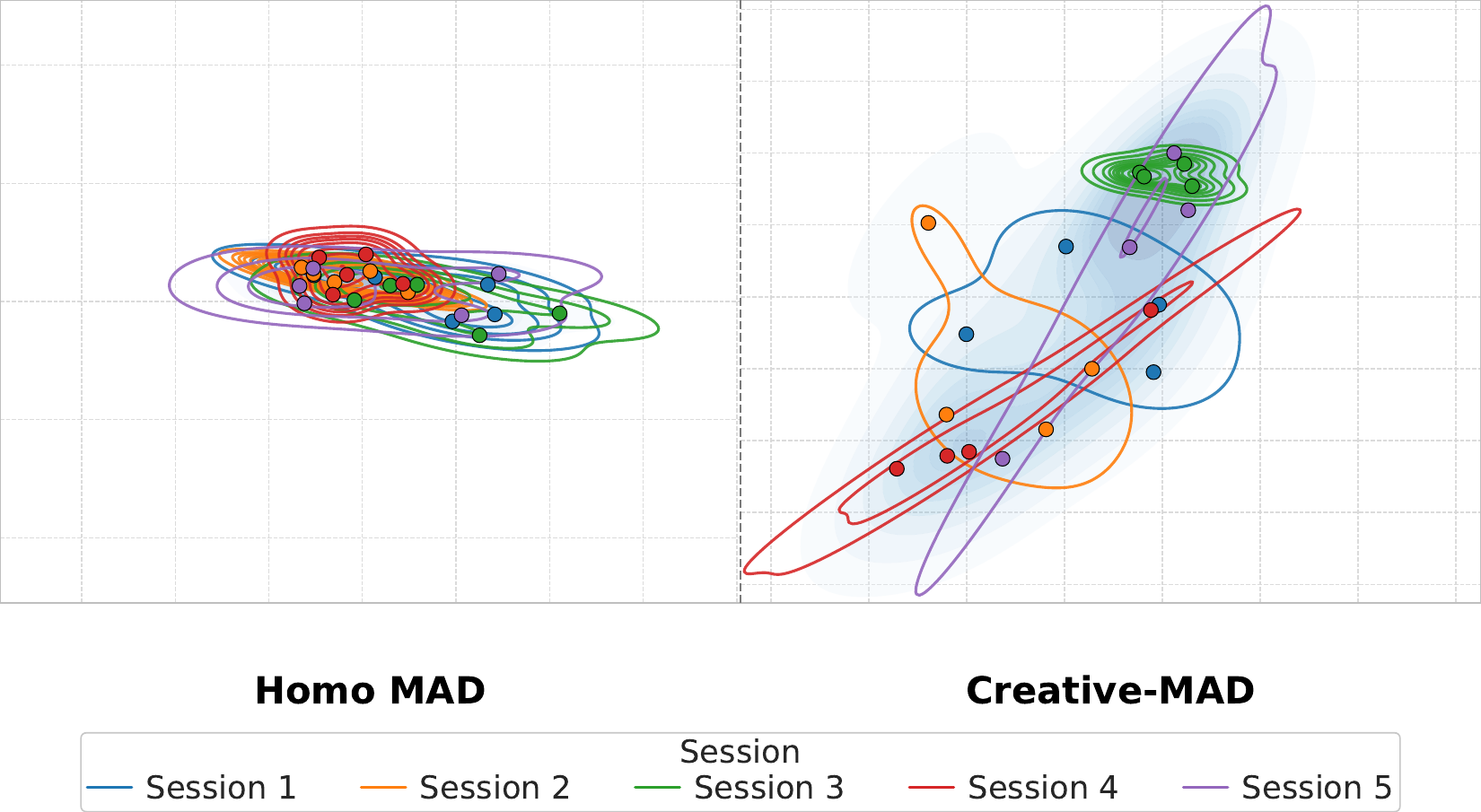}
    \caption{Inter-session diversity topography of Homo MAD vs. Creative-MAD (5 sessions, PCA projection). Colors denote individual sessions; contours indicate kernel density estimates.}
    
    \label{fig:inter_topography}
\end{figure}

Figure~\ref{fig:inter_topography} visualizes the semantic space of final outputs across five independent sessions via PCA projection, providing geometric evidence for Proposition~\ref{prop:intra_inter}. Under Homo MAD, outputs from all sessions collapse into a single dense, heavily overlapping region, confirming that convergence-driven debate forces independent sessions to occupy the same narrow semantic space. Conversely, Creative-MAD exhibits a markedly different topography: outputs within each session are highly distributed, with distinct sessions occupying separated regions of the semantic space. This empirically validates that preserving intra-session diversity directly unlocks higher inter-session diversity.

\subsection{Ablation Study}

\begin{figure}[t]
    \centering
    \includegraphics[width=\columnwidth]{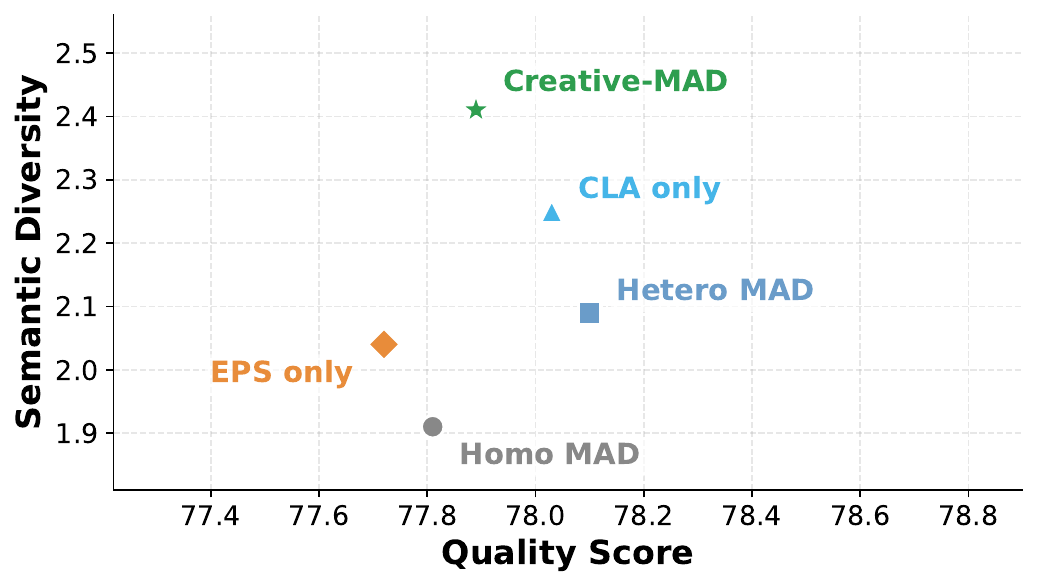}
    \caption{Semantic diversity versus quality score for Creative-MAD ablation variants and baselines on Qwen3-8B, averaged across four benchmarks.}
    
    \label{fig:component_ablation}
\end{figure}

Figure~\ref{fig:component_ablation} compares semantic diversity against quality for Creative-MAD ablation variants and baselines using Qwen3-8B, averaged across four benchmarks. All methods maintain comparable quality scores, confirming that diversity gains do not compromise output quality. CLA only yields a substantial diversity improvement over Homo MAD and consistently outperforms Hetero MAD. This demonstrates that anchoring agents to distinct cognitive modes is more effective at preserving diversity during debate than standard domain persona assignment. EPS only similarly improves diversity, confirming that filtering peer signals reduces majority pull even without a stable cognitive identity. Combining both components yields the largest diversity gains, demonstrating a clear synergy: CLA injects distinct perspectives at initialization, while EPS slows within-session convergence via selective peer exposure. Together, they effectively preserve the intra-session diversity required for inter-session gains.

\subsection{Sensitivity to $k$}

\begin{figure}[t]
    \centering
    \includegraphics[width=\columnwidth]{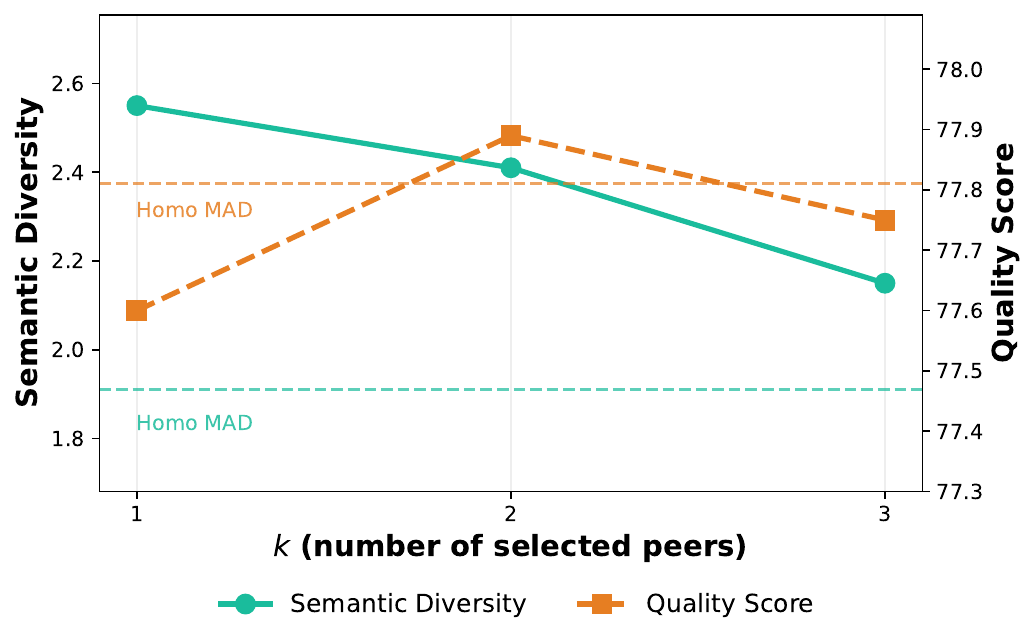}
    \caption{Sensitivity of Creative-MAD to the peer selection hyperparameter $k$ on Qwen3-8B.}
    
    \label{fig:sensitivity_k}
\end{figure}

Figure~\ref{fig:sensitivity_k} illustrates how Creative-MAD performs across different values of the peer selection hyperparameter $k$. Semantic diversity decreases monotonically as $k$ increases, which is expected: larger $k$ exposes each agent to more peers, gradually approaching the full peer pool of standard MAD and reintroducing majority pull. Conversely, quality score peaks at $k=2$ and drops slightly at $k=3$, while dropping noticeably at $k=1$, suggesting that a single peer provides insufficient context for meaningful debate. Taken together, $k=2$ achieves the best balance between diversity preservation and quality, and we adopt this value in all Creative-MAD experiments. Notably, even at $k=3$, Creative-MAD maintains higher diversity than Homo MAD, demonstrating that the method is robust to the choice of $k$ within a reasonable range.

\subsection{Peer Selection Strategy in EPS}

\begin{table}[h]
\centering
\caption{Effect of peer selection strategy on Qwen3-8B. Absolute quality and semantic diversity scores are reported as mean $\pm$ std.}
\small
\resizebox{\columnwidth}{!}{%
\begin{tabular}{l cccc}
\toprule
& \multicolumn{2}{c}{\textbf{LiveIdea}} 
& \multicolumn{2}{c}{\textbf{MacGyver}} \\
\cmidrule(lr){2-3} \cmidrule(lr){4-5}
\textbf{Method} & Quality & Diversity & Quality & Diversity \\
\midrule
Closest-$k$  & 72.8$\pm$0.2 & 2.75$\pm$0.04 & 76.8$\pm$0.2 & 1.52$\pm$0.05 \\
Random-$k$   & 72.1$\pm$0.3 & 2.88$\pm$0.06 & 77.0$\pm$0.4 & 1.59$\pm$0.05 \\
EPS          & 72.6$\pm$0.1 & 3.04$\pm$0.06 & 76.4$\pm$0.3 & 1.67$\pm$0.04 \\
\bottomrule
\end{tabular}%
}
\label{tab:eps_selection}
\end{table}

To isolate the contribution of semantic distance in EPS, we compare three peer selection strategies under the same Creative-MAD framework with $k=2$ on LiveIdeaBench and MacGyver: EPS selects the $k$ most semantically distant peers, Random-$k$ selects $k$ peers uniformly at random, and Closest-$k$ selects the $k$ most semantically similar peers. Table~\ref{tab:eps_selection} reveals a consistent hierarchy in semantic diversity across both benchmarks: $\text{Closest-}k$ yields the lowest diversity, $\text{Random-}k$ moderately improves it, and EPS achieves the highest, validating that maximizing semantic distance drives substantial diversity gains beyond mere pool reduction. Meanwhile, quality scores remain statistically invariant across all strategies, confirming that driving cross-perspective divergence does not compromise response quality.

\subsection{Additional Results}
Due to space constraints, we report the following results in the Appendix. Appendix~\ref{appendix:generalizability} validates generalizability across broader model scales and families, while Appendix~\ref{appendix:reliability} confirms evaluation reliability through inter-judge agreement, detailed human annotation studies, and robustness to embedding model choice. Appendix~\ref{appendix:hyperparameter_sensitivity} examines sensitivity to the number of debate rounds and the number of agents, and Appendix~\ref{appendix:multimodel} extends our analysis to model-heterogeneous settings. Appendix~\ref{appendix:protocol_robustness} evaluates robustness under an alternative debate protocol, and Appendix~\ref{appendix:cost_latency} quantifies Creative-MAD's computational overhead relative to standard MAD. Finally, Appendix~\ref{appendix:failure_cases} analyzes qualitative failure cases where Creative-MAD's diversity gains come at the cost of relevance or constraint satisfaction, and Appendix~\ref{appendix:qualitative} presents representative successful outputs across all four benchmarks.

\section{Conclusion}

In this work, we formalize the quality-diversity trade-off when applying MAD to creative tasks, proving that preserving intra-session diversity is a necessary condition for sustaining inter-session diversity. To address this, we propose Creative-MAD, which combines CLA to counter identity drift and EPS to mitigate majority pull. Extensive experiments confirm that Creative-MAD substantially improves cross-run diversity while maintaining output quality comparable to standard MAD.

\section*{Limitations}
Our work has several limitations that point to promising directions for future research. First, while our judge reliability analysis demonstrates agreement with human raters and commercial models in Appendix~\ref{appendix:reliability}, LLM-based evaluation remains an imperfect proxy for human aesthetic judgment in creative tasks, as automated judges may favor conventionally well-formed outputs over genuinely creative ones, potentially undervaluing subtle dimensions of creative quality. Second, since creative tasks admit no ground-truth answer, we employ a judge-based consensus mechanism to select the final output. While our theoretical analysis motivates CLA and EPS as inference-time interventions targeting identity drift and majority pull, the same analysis points to other directions worth exploring: training-time interventions, such as fine-tuning agents with a reward that encourages sustained divergence across debate rounds, could reduce reliance on a fixed cognitive lens; alternatively, redesigning the consensus stage itself to retain a diverse subset of candidates, rather than collapsing each session into a single selected output, could preserve more of the diversity that debate otherwise loses at consensus. Whether such alternative consensus strategies interact differently with diversity-preserving debate mechanisms remains an open question for future work. Finally, our analysis focuses on the fixed-prompt setting where the query $x$ is held constant across sessions to isolate diversity suppression arising from debate dynamics themselves; run-level interventions such as per-run persona resampling~\cite{hu2025debate} represent a complementary direction, and combining such approaches with Creative-MAD remains a promising avenue for future work.

\section*{Ethical Considerations}
The datasets used are publicly available benchmarks and do not contain personally identifiable information or sensitive content. The human annotation study (Appendix~\ref{appendix:human_agreement}) involved internal researchers who participated voluntarily, and no sensitive data was collected from participants. The proposed Creative-MAD framework generates open-ended creative text and does not target any specific demographic group or enable harmful applications. Automated evaluation via LLMaaJ may reflect biases of the underlying judge model; we mitigate this concern through inter-judge agreement analysis and human evaluation (Appendix~\ref{appendix:reliability}).

\section*{Acknowledgements}
AI assistants were used for language editing (e.g., grammar and phrasing checks) and for coding support during the preparation of this work.

\bibliography{custom}

\appendix
\clearpage
\section{Experimental Details}

\subsection{Dataset Details}
\label{appendix:dataset_details}

In this section, we provide a comprehensive description of the datasets used in our creative generation experiments and the specific number of samples utilized for each.

\textbf{LiveIdeaBench} \cite{ruan2026evaluating} is a specialized benchmark designed to evaluate the divergent thinking capabilities of LLMs in the context of scientific innovation. Instead of providing rich contextual information, this dataset utilizes minimal single-keyword prompts across 22 scientific domains to assess the model's ability to generate novel and feasible scientific hypotheses from its internal knowledge representations. We randomly subsample 300 keywords from the original 1,180 high-impact scientific keywords for our study.

\textbf{Argumentative Essays} (adapted from CLEAR \cite{huber-niklaus-2025-clear}) utilizes topics derived from the \textit{Argument Annotated Essays 2.0} corpus. While the original CLEAR framework is designed for the task of Argument Improvement (rewriting student essays), we specifically extract the controversial essay topics to evaluate the models' capacity for original argumentative generation. This task requires the model to construct structured, persuasive, and coherent arguments from scratch based on a given prompt. We utilize a random subset of 300 essay topics for this evaluation.

\textbf{MacGyver} \cite{tian2024macgyver} challenges models with creative physical problem-solving in real-world scenarios. The dataset is specifically curated to trigger out-of-the-box thinking by presenting problems that necessitate the unconventional usage of everyday objects, thereby pushing against the cognitive bias of ``functional fixedness.'' Models must devise physically feasible and efficient plans to achieve complex goals under strict constraints. For our experiments, we specifically filter for solvable samples—defined as problems for which a safe and reasonable solution can be achieved using the provided tools—and employ a random subset of 300 problem scenarios from this portion of the dataset.

\textbf{Arena Hard v2.0} \cite{li2025from} (creative writing subset) is a benchmark for open-ended creative text generation, featuring diverse prompts that require models to produce poems, lyrics, dialogues, and other imaginative responses under varied stylistic constraints. Unlike benchmarks centered on scientific ideation or physical problem-solving, this subset focuses on general creative writing and is well-suited for evaluating creativity, prompt relevance, and language quality in the absence of a single reference answer. For our experiments, we use the full creative writing subset of Arena Hard v2.0, consisting of 250 prompts.

\subsection{Hyperparameters \& Compute Resources}
\label{appendix:hyperparameters}

\paragraph{Hyperparameters.}
All MAD-based methods use $N = 5$ agents, $R = 2$ debate rounds, temperature $= 1.0$, and maximum output length of 2048 tokens. Creative-MAD uses $k = 2$ for EPS. The judge runs at temperature $= 0.2$.

\paragraph{Compute Resources.}
All experiments are conducted on NVIDIA H100 80GB GPUs using vLLM engine \cite{kwon2023efficient} with bf16 precision.

\section{Proof of Proposition~\ref{prop:intra_inter}}
\label{appendix:proof}
We consider the standard MAD setting in which the $G$ sessions are run independently with no shared state or communication, the input prompt is fixed across runs, and temperature sampling is the only source of stochasticity in the single-model setting. The proof operates solely on the output similarity kernel $K$, making no assumptions about whether outputs are generated by the same or different model families, and therefore extends to model-heterogeneous settings. Alternative interventions such as prompt perturbation or temperature scheduling are outside the scope of this setting.

When intra-session diversity decreases through debate, agents within each session $g$ produce increasingly similar outputs $\{y_i^{(R)}\}_{i=1}^{N}$, reducing the effective diversity of the pool from which the consensus mechanism selects $\hat{y}_g$. As a result, the final outputs $\{\hat{y}_g\}_{g=1}^{G}$ become semantically closer to one another across sessions, causing the off-diagonal entries of $K$ to increase and pushing $\bar{K}$ toward a uniform matrix:
\begin{equation}
    \bar{K} \to \frac{1}{G}\mathbf{1}\mathbf{1}^\top
    \implies H(K) \to 0
    \implies \mathcal{D}_{\mathrm{inter}} \to 1
\end{equation}
More generally, this relationship is monotonic: as intra-session diversity decreases, $K$ approaches a rank-1 structure with a single dominant eigenvalue $\lambda_1 \approx 1$ and remaining eigenvalues $\lambda_i \approx 0$, minimizing $H(K) = -\sum_i \lambda_i \log \lambda_i$ and consequently $\mathcal{D}_{\mathrm{inter}}$.

\paragraph{Practical Implications.} Proposition~\ref{prop:intra_inter} has two important practical implications.

\noindent\textbf{First, the consensus mechanism cannot compensate for low intra-session diversity.} When agents within session $g$ produce highly similar outputs $\{y_i^{(R)}\}_{i=1}^{N}$, the consensus mechanism, regardless of its selection criterion, must draw its final output $\hat{y}_g$ from a concentrated pool. Since this holds for every session $g$, the final outputs $\{\hat{y}_g\}_{g=1}^{G}$ cluster in similar regions across sessions, yielding low inter-session diversity regardless of the consensus strategy employed.

\noindent\textbf{Second, intra-session diversity is the primary controllable lever for inter-session diversity.} The exploration of session $g$ is entirely determined by its own internal debate dynamics, making diversity preservation during the debate process itself the only effective intervention for improving inter-session diversity, such that each session explores a broader region of $\mathcal{Y}$ and the selected output $\hat{y}_g$ can vary meaningfully across sessions.

\section{Martingale Inapplicability in Creative Settings}
\label{appendix:martingale}
\begin{remark}[Martingale Inapplicability]
The martingale result in \citet{choi2026debate} relies on two conditions: (1) agents 
generate responses from a finite categorical space $y \in \{1,\ldots,K\}$, and (2) belief 
updates follow Bayesian conjugacy over discrete counts under a Dirichlet-Compound-Multinomial 
(DCM) model. Creative tasks violate both conditions. First, the response space is open-ended 
($y \in \mathcal{T}^*$), rendering the marginal probability $P(y=k \mid \alpha)$ undefined. 
Second, quality is assessed via a continuous signal $q_{i,t} = \texttt{Judge}(y_{i,t}) \in 
\mathbb{R}$, and the agent update rule becomes:
\begin{equation}
    y_i^{(t)} = \mathcal{M}\left(x, 
    \mathcal{P}_i^{(t-1)}\right), \quad 
    q_{i,t} = \texttt{Judge}(y_i^{(t)})
\end{equation}
which is no longer a Bayesian count update. Consequently, the martingale condition 
$\mathbb{E}[q_{i,t} \mid q_{i,t-1}] = q_{i,t-1}$ need not hold, allowing debate to provide 
genuine quality improvement in creative settings.
\end{remark}

\section{Pairwise Win Rate Heatmaps}
Figure~\ref{fig:winrate_heatmap} presents the full pairwise win rate heatmaps across all method pairs on Qwen3-8B, where each cell $(i, j)$ denotes the win rate of method $i$ against method $j$.

\begin{figure*}[h]
    \centering
    \includegraphics[width=\linewidth]{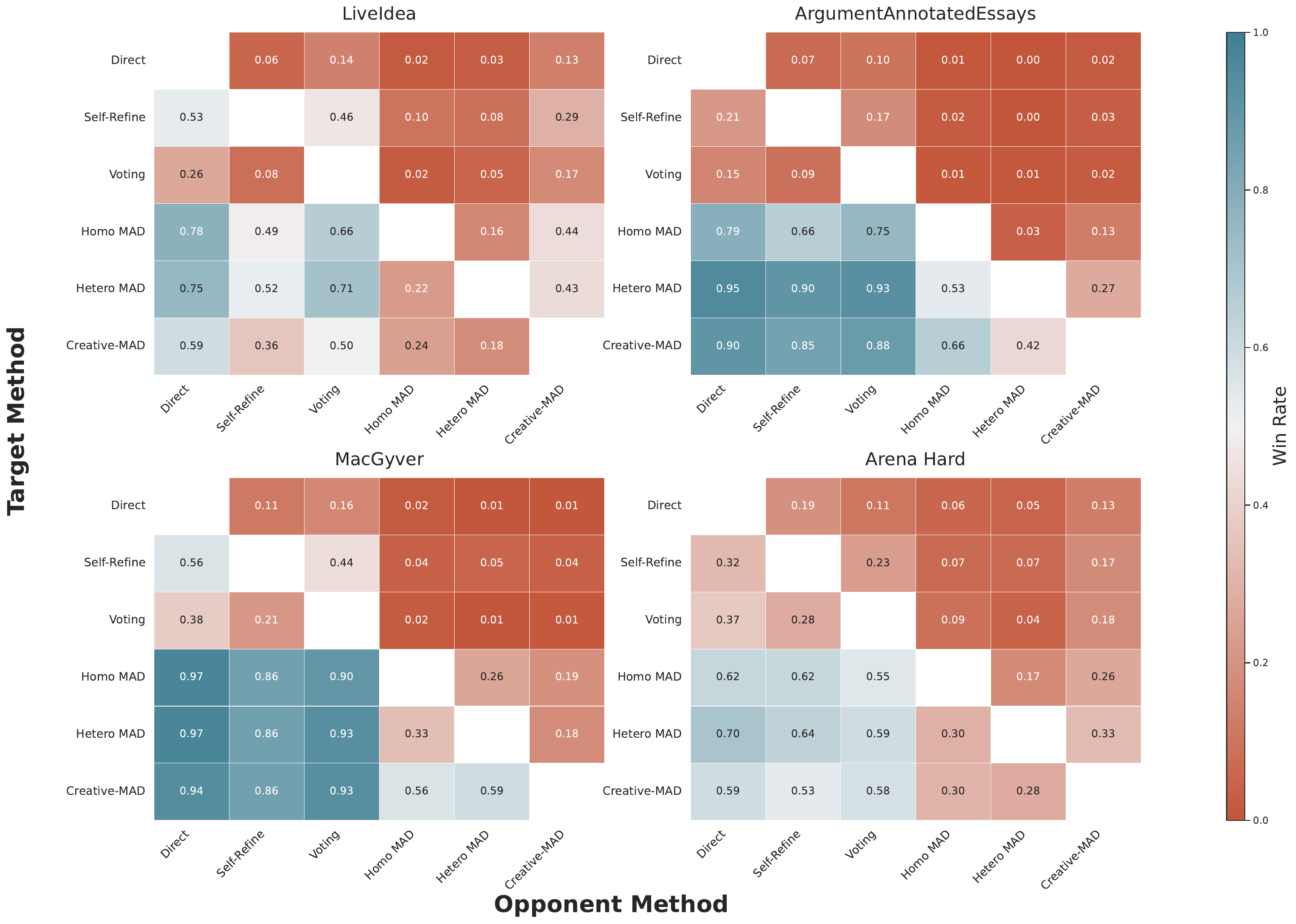}
    \caption{Pairwise win rate heatmaps across four benchmarks on Qwen3-8B. Each cell $(i, j)$ denotes the win rate of the target method (row) against the opponent method (column).}
    \label{fig:winrate_heatmap}
\end{figure*}

\section{Prompt Templates}

\subsection{MAD Templates}

The following template is used at each debate round $t \geq 1$ to prompt each agent to refine its response based on peer outputs.

\begin{tcolorbox}[
    colback=gray!5,
    colframe=gray!50,
    title=Debate Prompt Template,
    fonttitle=\small\bfseries,
    fontlower=\small,
    breakable
]
\texttt{Task: \{query\}}

\texttt{These are the recent ideas from other agents:}\\
\texttt{\{peer\_answers\}}

\texttt{This was your most recent idea:}\\
\texttt{\{own\_answer\}}

Within \texttt{<critique></critique>} tags, briefly evaluate the other ideas, filtering out clichés and synthesizing their best elements into a stronger concept. Then within \texttt{<answer></answer>} tags, write an improved and highly original response built directly upon this synthesis. Your answer must NOT reference other agents, the debate process, or include any meta-commentary.

Respond in the following format:
\texttt{<critique>}\\
\texttt{...}\\
\texttt{</critique>}\\
\texttt{<answer>}\\
\texttt{...}\\
\texttt{</answer>}
\end{tcolorbox}

\subsection{Consensus Selection Templates}

The following template is used at the consensus stage to select the best response from all agents' final outputs.

\begin{tcolorbox}[
    colback=gray!5,
    colframe=gray!50,
    title=Consensus Selection Prompt Template,
    fonttitle=\small\bfseries,
    fontlower=\small,
    breakable
]
\texttt{Task: \{query\}}

\texttt{Several experts have debated and submitted their final refined answers below:}\\
\texttt{\{all\_answers\}}

As an impartial and expert judge, evaluate these answers based on their creativity, logical coherence, and alignment with the task. Select the single best answer.

\textbf{IMPORTANT:} You must output ONLY the Agent Number of the winning answer enclosed in \texttt{<winner>} tags. For example, if Agent 2 is the best, output exactly:\\
\texttt{<winner>2</winner>}

Do NOT output the text of the answer itself, and do NOT provide any reasoning or extra text.
\end{tcolorbox}

\subsection{Cognitive Lens Templates}
\label{appendix:cognitive_lenses}

Table~\ref{tab:cognitive_lenses} presents the five cognitive lens templates assigned to agents in Creative-MAD. Each lens is assigned deterministically to one agent and persists across all debate rounds.

\begin{table*}[t]
\centering
\small
\caption{Cognitive lens templates used in Creative-MAD. Each agent is assigned one lens that persists throughout all debate rounds.}
\label{tab:cognitive_lenses}
\begin{tabular}{p{2cm} p{12cm}}
\toprule
\textbf{Lens} & \textbf{Template} \\
\midrule
Analytical & You process all information through an analytical and logical lens. When approaching any task, you identify underlying structures, patterns, and evidence. You ask: what logic connects the elements, what structure can be revealed, what reasoning makes this idea compelling and coherent? \\
\midrule
Emotional & You process all information through an emotional and human-centered lens. When approaching any task, you focus on human impact, feelings, values, and lived experience. You ask: how does this affect people deeply, what emotions does it evoke, what human truth or vulnerability does it reveal? \\
\midrule
Critical & You process all information through a critical and adversarial lens. When approaching any task, you interrogate assumptions, surface hidden contradictions, and find the perspective that challenges the obvious. You ask: what is being overlooked, what tension or paradox exists here, what would a skeptic or contrarian notice? \\
\midrule
Analogical & You process all information through an analogical and cross-domain lens. When approaching any task, you draw unexpected connections to distant domains, metaphors, and parallel systems. You ask: what surprising analogy illuminates this, what can be borrowed from an unrelated field, what deeper pattern connects seemingly unrelated things? \\
\midrule
Practical & You process all information through a practical and implementation-focused lens. When approaching any task, you evaluate real-world feasibility, concrete mechanisms, and actionable outcomes. You ask: what actually works in practice, how would this be realized step by step, what real-world constraints and opportunities shape the best solution? \\
\bottomrule
\end{tabular}
\end{table*}

\subsection{Persona Prompt Templates}
\label{appendix:persona_prompts}
Table~\ref{tab:persona_prompts} presents the five persona prompt templates used in Hetero MAD. Each persona is assigned deterministically to one agent and persists across all debate rounds.

\begin{table*}[t]
\centering
\small
\caption{Persona prompt templates used in Hetero MAD. Each agent is assigned one persona that persists throughout all debate rounds.}
\label{tab:persona_prompts}
\begin{tabular}{p{2cm} p{12cm}}
\toprule
\textbf{Persona} & \textbf{Template} \\
\midrule
Economist & You are an economist. You are good at economics, finance, and business. You have experience on understanding charts while interpreting the macroeconomic environment prevailing across world economies. \\
\midrule
Psychologist & You are a psychologist. You are good at psychology, sociology, and philosophy. You give people scientific suggestions that will make them feel better. \\
\midrule
Lawyer & You are a lawyer. You are good at law, politics, and history. \\
\midrule
Doctor & You are a doctor and come up with creative treatments for illnesses or diseases. You are able to recommend conventional medicines, herbal remedies and other natural alternatives. You also consider the patient's age, lifestyle and medical history when providing your recommendations. \\
\midrule
Historian & You are a historian. You research and analyze cultural, economic, political, and social events in the past, collect data from primary sources and use it to develop theories about what happened during various periods of history. \\
\bottomrule
\end{tabular}
\end{table*}

\subsection{LLMaaJ Evaluation Rubrics Templates}
\label{appendix:rubrics}

For each benchmark, we employ a dataset-specific rubric to evaluate instance-level quality. All rubrics use a strict 1--10 scale per dimension, and the final quality score is computed as the average across dimensions.

\subsubsection{LiveIdeaBench}
\label{appendix:rubric_liveidea}

The LiveIdeaBench rubric evaluates scientific ideation quality across three dimensions: originality, feasibility, and clarity, adapted from \citet{ruan2026evaluating} with a 1--10 scoring scale.

\begin{tcolorbox}[
    colback=gray!5,
    colframe=gray!50,
    title=LiveIdeaBench Evaluation Rubric,
    fonttitle=\small\bfseries,
    fontlower=\small,
    breakable
]
You are a cynical and world-class expert evaluator of scientific innovation. You are bored by standard AI-generated summaries and are looking for genuinely groundbreaking frontiers.

The user's query provides a KEYWORD. The proposed idea MUST address a meaningful problem related to that keyword with high ``Practical Ingenuity''.

You will assess the response on three key dimensions using a STRICT 1--10 scale:

\begin{enumerate}
    \item \textbf{Originality:} Does it offer a novel, non-obvious contribution? Scores of 7+ are reserved for ideas that go beyond standard textbook knowledge or common AI patterns.
    \item \textbf{Feasibility:} Scientific validity and technical implementation. A brilliant idea that violates physics or logic should be penalized here.
    \item \textbf{Clarity:} Structural elegance and the ability to communicate complex innovation simply.
\end{enumerate}

\textbf{Scoring Guidelines (Be Strict):}
\begin{itemize}
    \item \textbf{1--3 (Poor):} Irrelevant, logically flawed, or purely superficial.
    \item \textbf{4--5 (Baseline):} Technically correct and safe, but ``standard'' and lacking any significant creative spark. This is the default for most AI-generated responses.
    \item \textbf{6--7 (Above Average):} Shows a specific, clever angle or an unconventional but viable application.
    \item \textbf{8--10 (Exceptional):} Truly ingenious, transformative, and non-obvious thinking that demonstrates professional-level scientific insight.
\end{itemize}

If the idea is completely irrelevant to the provided keyword, assign a score of 1--2 across all dimensions.

\textbf{Output format:}
\begin{verbatim}
{
    "analysis": "<brief analysis, specifically explaining 
                  why the idea is or isn't innovative>",
    "scores": {
        "originality": <int>,
        "feasibility": <int>,
        "clarity": <int>
    }
}
\end{verbatim}
\end{tcolorbox}

\subsubsection{Argumentative Essays}
\label{appendix:rubric_aae}

The Argument Annotated Essays corpus is originally designed for argument rewriting evaluation~\citep{huber-niklaus-2025-clear}. Since we repurpose it for argumentative generation from scratch, we design a rubric tailored to original argument quality rather than rewriting criteria, evaluating three dimensions: persuasiveness, novelty, and logical coherence.

\begin{tcolorbox}[
    colback=gray!5,
    colframe=gray!50,
    title=Argumentative Essays Evaluation Rubric,
    fonttitle=\small\bfseries,
    fontlower=\small,
    breakable
]
You are an expert debate judge and rhetorician evaluating argument generation. You will assess the response on three key dimensions:

\begin{enumerate}
    \item \textbf{Persuasiveness:} How compelling, well-supported, and convincing the fundamental argument is.
    \item \textbf{Novelty:} Originality of the angle, avoiding typical clichés or surface-level reasoning.
    \item \textbf{Logical Coherence:} Clarity, structure, and the soundness of the logical deduction.
\end{enumerate}

Provide your brief analysis (less than 100 words) of the argument. Then, inside a ``scores'' object, provide a score from 1 to 10 for each dimension where 1--3 = poorly reasoned, 4--6 = standard argument, 7--10 = highly compelling and insightful.

\textbf{Output format:}
\begin{verbatim}
{
    "analysis": "<your brief analysis here>",
    "scores": {
        "persuasiveness": <score>,
        "novelty": <score>,
        "logical_coherence": <score>
    }
}
\end{verbatim}
\end{tcolorbox}

\subsubsection{MacGyver}
\label{appendix:rubric_macgyver}

MacGyver evaluates creative physical problem-solving under strict 
resource constraints~\citep{tian2024macgyver}. Following the 
evaluation criteria described in the original benchmark, we design 
a rubric assessing three dimensions: out-of-the-box thinking, 
practical feasibility, and constraint utilization.

\begin{tcolorbox}[
    colback=gray!5,
    colframe=gray!50,
    title=MacGyver Evaluation Rubric,
    fonttitle=\small\bfseries,
    fontlower=\small,
    breakable
]
{\small
You are an expert evaluator assessing creative problem-solving 
in extreme constraint scenarios. You will assess the proposed 
solution on three key dimensions:

\begin{enumerate}
    \item \textbf{Out-of-the-box Thinking:} How unconventional 
    and clever the solution is, avoiding standard generic answers.
    \item \textbf{Practical Feasibility:} Whether the physical 
    actions described would actually work and obey the laws of 
    physics.
    \item \textbf{Constraint Utilization:} How effectively the 
    limited available items are strictly used without assuming 
    extra resources.
\end{enumerate}

Provide your brief analysis (less than 100 words) of the solution. 
Then, inside a ``scores'' object, provide a score from 1 to 10 
for each dimension where 1--3 = hallucinatory/boring, 
4--6 = works but basic, 7--10 = highly ingenious and viable.

\textbf{Output format:}
\begin{verbatim}
{"analysis": "<your brief analysis here>",
 "scores": {"out_of_the_box_thinking": <score>,
            "practical_feasibility": <score>,
            "constraint_utilization": <score>}}
\end{verbatim}
}
\end{tcolorbox}

\subsubsection{Arena Hard v2.0}
\label{appendix:rubric_arenahard}

Arena Hard v2.0 (creative writing subset) covers diverse creative writing tasks including poetry, dialogue, and genre writing. Following the evaluation criteria of the benchmark, we design a rubric assessing three dimensions: creativity, relevance, and language quality.

\begin{tcolorbox}[
    colback=gray!5,
    colframe=gray!50,
    title=Arena Hard v2.0 Evaluation Rubric,
    fonttitle=\small\bfseries,
    fontlower=\small,
    breakable
]
{\small
You are an expert judge evaluating creative writing responses. 
You will assess the response on three key dimensions using a 
STRICT 1--10 scale:

\begin{enumerate}
    \item \textbf{Creativity:} The response demonstrates 
    originality, unique ideas, and avoids clichés or generic 
    patterns.
    \item \textbf{Relevance:} The response directly addresses 
    and fulfills the requirements of the prompt.
    \item \textbf{Language Quality:} The writing is fluent, 
    coherent, and stylistically appropriate for the genre 
    (e.g., poetry, rap, dialogue, etc.).
\end{enumerate}

\textbf{Scoring Guidelines (Be Strict):}
\begin{itemize}
    \item \textbf{1--3 (Poor):} Off-topic, generic, or poorly 
    written. Lacks creativity or fails to address the prompt.
    \item \textbf{4--5 (Baseline):} Adequate and relevant, but 
    not particularly creative or polished. May have minor 
    language issues.
    \item \textbf{6--7 (Above Average):} Shows clear creativity, 
    good relevance, and strong language use.
    \item \textbf{8--10 (Exceptional):} Highly original, 
    perfectly fits the prompt, and is masterfully written.
\end{itemize}

If the response is completely irrelevant or empty, assign a 
score of 1--2 across all dimensions.

\textbf{Output format:}
\begin{verbatim}
{"analysis": "<brief analysis of strengths and weaknesses>",
 "scores": {"creativity": <int>,
            "relevance": <int>,
            "language_quality": <int>}}
\end{verbatim}
}
\end{tcolorbox}

\subsection{Pairwise Win Rate Evaluation Templates}
\label{appendix:pairwise}

For pairwise win rate evaluation, we employ the same 
Qwen3.5-397B-A17B judge with dataset-specific criteria 
as defined in the rubrics above (Appendix~\ref{appendix:rubrics}). 
For each query, we sample one response per method and perform 
pairwise comparisons across all $\binom{K}{2}$ method pairs, 
where $K$ is the number of methods. To mitigate positional bias, 
each pair is evaluated twice with response positions swapped. 
A winner is declared only when both runs agree; disagreements 
are counted as ties. The win rate of each method is computed as 
the fraction of matches won across all pairwise comparisons and 
all queries.

\begin{tcolorbox}[
    colback=gray!5,
    colframe=gray!50,
    title=Pairwise Win Rate Judge Prompt,
    fonttitle=\small\bfseries,
    fontlower=\small,
    breakable
]
{\small
You are an expert evaluator assessing AI-generated responses 
based on ``Creative Quality''.

Determine which response is better based on how well it 
satisfies the following criteria specific to the task domain:

\texttt{\{criteria\}}

You will be given a task and two AI responses (Response A 
and Response B). Determine which response demonstrates 
superior quality according to the criteria above.

You MUST output a valid JSON object with exactly two keys:
\begin{enumerate}
    \item \textbf{reasoning:} A brief comparison (2--3 sentences) 
    explaining the key differences. Before choosing the winner, 
    explicitly list at least one specific point where the winner 
    outperformed the other.
    \item \textbf{winner:} One of ``A'', ``B'', or ``tie''.
\end{enumerate}

A ``tie'' is STRONGLY DISCOURAGED and is only acceptable if 
and only if both responses are perfectly identical in value. 
Even a minor improvement or a better execution means that 
model must win.

\textbf{Output format:}
\begin{verbatim}
{"reasoning": "<brief comparison>",
 "winner": "A" | "B" | "tie"}
\end{verbatim}
}
\end{tcolorbox}

\section{Model Generalizability}
\label{appendix:generalizability}

\begin{table*}[t]
\centering
\small
\caption{Generalizability results across four creative benchmarks on additional model scales. We report instance-level quality score (mean $\pm$ std) and set-level semantic diversity score (mean $\pm$ std). Higher is better. \textbf{Bold} indicates the best result(s); multiple entries may be bolded if not statistically significantly different (paired $t$-test, $p < 0.05$).}
\label{tab:generalizability}
\resizebox{\textwidth}{!}{%
\begin{tabular}{l cc cc cc cc}
\toprule
\multirow{2}{*}{\textbf{Method}} 
& \multicolumn{2}{c}{\textbf{LiveIdea}} 
& \multicolumn{2}{c}{\textbf{AAE}} 
& \multicolumn{2}{c}{\textbf{MacGyver}} 
& \multicolumn{2}{c}{\textbf{AH v2.0}} \\
\cmidrule(lr){2-3} \cmidrule(lr){4-5} \cmidrule(lr){6-7} \cmidrule(lr){8-9}
& Quality & Diversity
& Quality & Diversity
& Quality & Diversity
& Quality & Diversity \\
\midrule
\multicolumn{9}{l}{\textit{Llama-3.2-3B-Instruct}} \\
\cmidrule(r){1-9}
Direct   & 61.4$\pm$0.3 & 3.19$\pm$0.04 & 70.0$\pm$0.4 & 1.68$\pm$0.04 & 48.6$\pm$0.4 & 1.79$\pm$0.06 & 47.2$\pm$0.4 & 2.61$\pm$0.05 \\
Self-Refine  & \textbf{64.5}$\pm$0.2 & 3.13$\pm$0.05 & \textbf{74.9}$\pm$0.6 & 1.85$\pm$0.04 & 49.4$\pm$0.7 & 1.78$\pm$0.04 & 54.1$\pm$0.8 & 2.62$\pm$0.06 \\
Voting       & 61.4$\pm$0.5 & 3.13$\pm$0.04 & 70.7$\pm$0.2 & 1.69$\pm$0.05 & 50.4$\pm$1.7 & 1.80$\pm$0.05 & 48.6$\pm$0.8 & 2.58$\pm$0.04 \\
Homo MAD     & 64.1$\pm$0.3 & 2.72$\pm$0.06 & \textbf{74.4}$\pm$0.3 & 1.51$\pm$0.06 & \textbf{52.9}$\pm$0.7 & 1.72$\pm$0.04 & \textbf{70.7}$\pm$3.9 & 2.44$\pm$0.06 \\
Hetero MAD   & \textbf{64.7}$\pm$0.3 & 3.22$\pm$0.05 & \textbf{74.1}$\pm$0.3 & 1.72$\pm$0.06 & \textbf{53.2}$\pm$0.6 & 1.80$\pm$0.05 & \textbf{69.1}$\pm$2.1 & 2.63$\pm$0.04 \\
Creative-MAD & \textbf{64.6}$\pm$0.2 & \textbf{3.71}$\pm$0.06 & \textbf{74.1}$\pm$0.2 & \textbf{2.23}$\pm$0.05 & \textbf{52.8}$\pm$0.6 & \textbf{2.20}$\pm$0.07 & \textbf{70.0}$\pm$2.5 & \textbf{3.08}$\pm$0.04 \\
\midrule
\multicolumn{9}{l}{\textit{Qwen3-32B}} \\
\cmidrule(r){1-9}
Direct   & 70.2$\pm$0.3 & 2.73$\pm$0.05 & 73.4$\pm$0.1 & 1.57$\pm$0.05 & 76.3$\pm$0.4 & 1.62$\pm$0.04 & 77.9$\pm$0.1 & 2.13$\pm$0.06 \\
Self-Refine  & 75.3$\pm$0.3 & 2.78$\pm$0.06 & 76.8$\pm$0.1 & 1.67$\pm$0.04 & 79.1$\pm$0.3 & 1.58$\pm$0.06 & 79.3$\pm$0.2 & 2.15$\pm$0.04 \\
Voting       & 71.5$\pm$0.3 & 2.79$\pm$0.04 & 74.3$\pm$0.1 & 1.57$\pm$0.06 & 78.5$\pm$0.2 & 1.62$\pm$0.05 & 82.2$\pm$0.2 & 2.19$\pm$0.04 \\
Homo MAD     & \textbf{77.2}$\pm$0.3 & 2.69$\pm$0.05 & 77.2$\pm$0.2 & 1.56$\pm$0.05 & \textbf{81.2}$\pm$0.3 & 1.59$\pm$0.06 & \textbf{83.9}$\pm$0.3 & 1.99$\pm$0.05 \\
Hetero MAD   & \textbf{76.8}$\pm$0.3 & 2.81$\pm$0.04 & \textbf{78.0}$\pm$0.2 & 1.62$\pm$0.04 & \textbf{81.2}$\pm$0.3 & 1.64$\pm$0.05 & \textbf{83.0}$\pm$0.3 & 2.19$\pm$0.06 \\
Creative-MAD & \textbf{77.0}$\pm$0.3 & \textbf{3.26}$\pm$0.06 & \textbf{78.0}$\pm$0.2 & \textbf{2.10}$\pm$0.07 & 80.0$\pm$0.3 & \textbf{1.91}$\pm$0.04 & \textbf{83.5}$\pm$0.3 & \textbf{2.62}$\pm$0.05 \\
\bottomrule
\end{tabular}%
}
\end{table*}

To verify that our findings generalize beyond the primary models evaluated in the main experiments, we conduct additional evaluations on two models spanning a wider range of parameter scales: Llama-3.2-3B-Instruct \cite{grattafiori2024llama} (3B) and Qwen3-32B \cite{yang2025qwen3} (32B). Together with the primary models (Qwen3-8B and Gemma-3-12B-Instruct), this covers four scales across three model families (Qwen, Llama, Gemma).

Table~\ref{tab:generalizability} reports quality and semantic diversity results across all four benchmarks. The quality-diversity trade-off observed in the main experiments holds consistently across both model scales: Homo MAD and Hetero MAD achieve higher quality than almost non-debate baselines while exhibiting lower diversity, and Creative-MAD maintains quality comparable to standard MAD variants while substantially improving semantic diversity across all benchmarks. These results confirm that the trade-off and our proposed solution are not artifacts of a specific model scale or family, but reflect a structural property of multi-agent debate dynamics in creative generation.

\section{Evaluation Reliability}
\label{appendix:reliability}

We validate the reliability of our evaluation along four complementary dimensions: agreement of our LLMaaJ with commercial judge models (Section~\ref{appendix:judge_agreement}), agreement with human annotators (Section~\ref{appendix:human_agreement}), robustness of our semantic diversity metric across different embedding models (Section~\ref{appendix:diversity_metric}), and robustness of quality rankings to length bias (Section~\ref{appendix:length_bias}).

\subsection{LLMaaJ Inter-Model Agreement}
\label{appendix:judge_agreement}

\begin{table*}[t]
\centering
\small
\caption{Inter-judge agreement on rubric-based absolute scores (mean $\pm$ std over $G=5$ outputs per method per query) on a subset of 50 queries per benchmark (Qwen3-8B outputs). Spearman $\rho$ is computed between Qwen3.5-397B-A17B (Qwen) and each commercial judge — GPT-5.4 (GPT) and Claude Sonnet 4.6 (Sonnet) — across methods per benchmark. Higher $\rho$ indicates stronger agreement.}
\label{tab:judge_agreement_score}
\resizebox{\textwidth}{!}{%
\begin{tabular}{l ccc ccc ccc ccc}
\toprule
& \multicolumn{3}{c}{\textbf{LiveIdea}}
& \multicolumn{3}{c}{\textbf{AAE}}
& \multicolumn{3}{c}{\textbf{MacGyver}}
& \multicolumn{3}{c}{\textbf{AH v2.0}} \\
\cmidrule(lr){2-4} \cmidrule(lr){5-7}
\cmidrule(lr){8-10} \cmidrule(lr){11-13}
\textbf{Method}
& \textbf{Qwen} & \textbf{GPT} & \textbf{Sonnet}
& \textbf{Qwen} & \textbf{GPT} & \textbf{Sonnet}
& \textbf{Qwen} & \textbf{GPT} & \textbf{Sonnet}
& \textbf{Qwen} & \textbf{GPT} & \textbf{Sonnet} \\
\midrule
Direct        & $66.5 \pm 0.3$ & $63.0 \pm 1.0$ & $59.8 \pm 0.0$ & $73.5 \pm 0.7$ & $59.6 \pm 0.1$ & $62.6 \pm 0.4$ & $64.6 \pm 1.6$ & $30.9 \pm 0.5$ & $26.9 \pm 0.8$ & $75.3 \pm 1.2$ & $58.7 \pm 0.1$ & $54.7 \pm 0.8$ \\
Self-Refine   & $69.7 \pm 0.5$ & $64.2 \pm 0.7$ & $60.1 \pm 1.0$ & $74.4 \pm 0.4$ & $60.3 \pm 0.2$ & $63.7 \pm 0.6$ & $69.1 \pm 1.2$ & $33.9 \pm 0.3$ & $30.3 \pm 1.7$ & $78.0 \pm 0.9$ & $60.9 \pm 0.1$ & $55.4 \pm 1.2$ \\
Voting        & $67.5 \pm 0.7$ & $63.8 \pm 0.5$ & $59.1 \pm 1.4$ & $74.4 \pm 0.3$ & $60.0 \pm 0.2$ & $63.2 \pm 0.6$ & $67.0 \pm 1.3$ & $32.3 \pm 1.5$ & $26.7 \pm 2.2$ & $78.8 \pm 0.4$ & $60.8 \pm 1.6$ & $53.6 \pm 0.7$ \\
Homo MAD      & $73.7 \pm 1.0$ & $65.0 \pm 0.2$ & $61.2 \pm 0.8$ & $80.6 \pm 0.3$ & $62.4 \pm 0.2$ & $70.3 \pm 1.0$ & $77.1 \pm 0.6$ & $36.2 \pm 0.5$ & $35.2 \pm 0.5$ & $81.4 \pm 0.3$ & $62.1 \pm 0.4$ & $57.6 \pm 2.0$ \\
Hetero MAD    & $72.7 \pm 0.4$ & $64.1 \pm 1.4$ & $61.8 \pm 1.4$ & $82.3 \pm 0.2$ & $63.7 \pm 0.2$ & $72.5 \pm 0.3$ & $78.4 \pm 0.2$ & $36.6 \pm 0.7$ & $36.4 \pm 0.4$ & $82.7 \pm 0.2$ & $62.4 \pm 0.5$ & $58.4 \pm 0.1$ \\
Creative-MAD  & $71.6 \pm 0.7$ & $65.7 \pm 0.3$ & $61.8 \pm 2.2$ & $82.1 \pm 0.2$ & $64.5 \pm 0.7$ & $79.1 \pm 0.7$ & $75.7 \pm 2.1$ & $36.8 \pm 2.1$ & $30.9 \pm 1.7$ & $80.8 \pm 0.5$ & $61.7 \pm 0.8$ & $57.1 \pm 0.8$ \\
\midrule
Spearman $\rho$ & ---          & 0.83         & 0.80         & ---          & 0.94         & 0.94         & ---          & 0.77         & 0.94         & ---          & 0.94         & 0.94 \\
\bottomrule
\end{tabular}%
}
\end{table*}

\begin{table*}[t]
\centering
\small
\caption{Inter-judge agreement on pairwise win rates (\%)  on a subset of 50 queries per benchmark (Qwen3-8B outputs,  1 random output per method per query). Spearman $\rho$ is computed between Qwen3.5-397B-A17B (Qwen) and each commercial judge — GPT-5.4 (GPT) and Claude Sonnet 4.6 (Sonnet) — across methods per benchmark. Higher $\rho$ indicates stronger agreement.}
\label{tab:judge_agreement_winrate}
\resizebox{\textwidth}{!}{%
\begin{tabular}{l ccc ccc ccc ccc}
\toprule
& \multicolumn{3}{c}{\textbf{LiveIdea}}
& \multicolumn{3}{c}{\textbf{AAE}}
& \multicolumn{3}{c}{\textbf{MacGyver}}
& \multicolumn{3}{c}{\textbf{AH v2.0}} \\
\cmidrule(lr){2-4} \cmidrule(lr){5-7}
\cmidrule(lr){8-10} \cmidrule(lr){11-13}
\textbf{Method}
& \textbf{Qwen} & \textbf{GPT} & \textbf{Sonnet}
& \textbf{Qwen} & \textbf{GPT} & \textbf{Sonnet}
& \textbf{Qwen} & \textbf{GPT} & \textbf{Sonnet}
& \textbf{Qwen} & \textbf{GPT} & \textbf{Sonnet} \\
\midrule
Direct        & 8.4  & 10.5 & 12.0 & 7.2  & 12.0 & 9.5  & 7.6  & 12.0 & 17.5 & 12.0 & 23.5 & 20.0 \\
Self-Refine       & 29.5 & 30.5 & 28.5 & 8.8  & 20.0 & 17.5 & 21.6 & 18.5 & 31.0 & 18.5 & 26.0 & 25.5 \\
Voting            & 9.5  & 15.0 & 16.5 & 6.4  & 11.5 & 12.0 & 11.6 & 21.0 & 16.0 & 21.0 & 30.0 & 21.0 \\
Homo MAD   & 45.2 & 52.0 & 66.0 & 59.0 & 43.6 & 57.0 & 70.0 & 62.0 & 54.0 & 54.0 & 42.0 & 34.0 \\
Hetero MAD & 43.2 & 48.2 & 44.5 & 72.5 & 64.8 & 70.0 & 75.5 & 60.0 & 63.0 & 63.0 & 39.0 & 42.5 \\
Creative-MAD & 41.6 & 50.1 & 64.4 & 80.4 & 75.2 & 79.8 & 77.2 & 65.2 & 59.2 & 66.2 & 41.5 & 39.2 \\
\midrule
Spearman $\rho$   & --   & 0.94 & 0.94 & --   & 1.00 & 0.94 & --   & 0.94 & 0.77 & --   & 0.83 & 0.89  \\
\bottomrule
\end{tabular}%
}
\end{table*}

\begin{table*}[t]
\centering
\small
\caption{Vendi Score across five embedding models on LiveIdeaBench (Qwen3-8B). Creative-MAD consistently ranks first and Homo MAD remains at the bottom across almost all embeddings, demonstrating a strong ranking consensus (average Spearman $\rho = 0.93$ vs.\ MiniLM-L6).}
\label{tab:vendi_stability}
\begin{tabular}{lccccc}
\toprule
\textbf{Method} & \textbf{MiniLM-L6} & \textbf{Jina-v3} & \textbf{Qwen3-Emb} & \textbf{BGE-base} & \textbf{Gemma-300M} \\
\midrule
Direct        & 2.5434 & 2.1987 & 2.3298 & 1.6583 & 1.9111 \\
Self-Refine       & 2.5921 & 2.2867 & 2.4408 & 1.6888 & 1.9748 \\
Voting            & 2.5846 & 2.2199 & 2.3362 & 1.6674 & 1.9398 \\
Homo MAD   & 2.4740 & 2.1726 & 2.3557 & 1.6256 & 1.8855 \\
Hetero MAD & 2.6063 & 2.2385 & 2.4498 & 1.6945 & 1.9589 \\
Creative-MAD      & 3.0444 & 2.6515 & 2.7573 & 1.8579 & 2.3647 \\
\midrule
Spearman $\rho$ & --- & 0.94 & 0.83 & 1.00 & 0.94 \\
\bottomrule
\end{tabular}
\end{table*}

\begin{table}[t]
\centering
\small
\caption{Effect of number of debate rounds $R$ on Homo MAD (Qwen3-8B, $N=5$ fixed). Quality denotes rubric-based absolute score (mean $\pm$ std) and Diversity denotes semantic diversity score (mean $\pm$ std). $\dagger$ denotes the default setting used in main experiments.}
\label{tab:sensitivity_rounds}
\begin{tabular}{c cccc}
\toprule
& \multicolumn{2}{c}{\textbf{LiveIdea}} 
& \multicolumn{2}{c}{\textbf{MacGyver}} \\
\cmidrule(lr){2-3} \cmidrule(lr){4-5}
\textbf{R} & Quality & Diversity & Quality & Diversity \\
\midrule
1           & 69.9$\pm$0.1 & 2.54$\pm$0.04 & 67.5$\pm$0.3 & 1.50$\pm$0.03 \\
2$^\dagger$ & 74.0$\pm$0.1 & 2.47$\pm$0.03 & 76.3$\pm$0.3 & 1.48$\pm$0.04 \\
3           & 75.8$\pm$0.2 & 2.43$\pm$0.03 & 78.0$\pm$0.3 & 1.45$\pm$0.05 \\
4           & 77.3$\pm$0.1 & 2.40$\pm$0.04 & 78.5$\pm$0.4 & 1.42$\pm$0.03 \\
5           & 78.0$\pm$0.1 & 2.38$\pm$0.03 & 79.4$\pm$0.3 & 1.40$\pm$0.04 \\
\bottomrule
\end{tabular}
\end{table}

While commercial models such as GPT-5.4 \cite{openai2026gpt54} and Claude Sonnet~4.6 \cite{claude2026sonnet} are widely adopted as judges in recent work, our evaluation involves a large number of inference calls across methods, benchmarks, and independent runs (15 runs per method per query for statistical reliability), making commercial API-based evaluation prohibitively expensive at scale. We therefore employ Qwen3.5-397B-A17B as our primary judge and validate its reliability against these commercial alternatives on a representative subset.

\paragraph{Setup.}
We randomly sample 50 queries per benchmark and evaluate Qwen3-8B outputs from all methods using all three judges. For rubric-based score validation (Table~\ref{tab:judge_agreement_score}), we compute the mean rubric score across $G=5$ randomly sampled outputs per method per query. For pairwise win rate validation (Table~\ref{tab:judge_agreement_winrate}), we use one randomly selected output per method per query. For each metric, we compute Spearman $\rho$ between Qwen3.5-397B-A17B and each commercial judge across methods per benchmark, measuring the consistency of method rankings induced by each judge.

\paragraph{Results.}
Tables~\ref{tab:judge_agreement_score} and~\ref{tab:judge_agreement_winrate} show generally strong agreement between Qwen3.5-397B-A17B and both commercial judges across most benchmarks. Spearman $\rho$ for rubric scores averages 0.87 and 0.91 against GPT-5.4 and Claude Sonnet~4.6 respectively, and 0.93 and 0.89 for pairwise win rates, suggesting that method rankings produced by our judge are largely consistent with those of commercial alternatives. These results support the use of Qwen3.5-397B-A17B as a reliable and cost-effective alternative for evaluating creative generation quality at scale.

\subsection{Human Agreement Study}
\label{appendix:human_agreement}

\begin{figure*}[h]
    \centering
    \includegraphics[width=\linewidth]{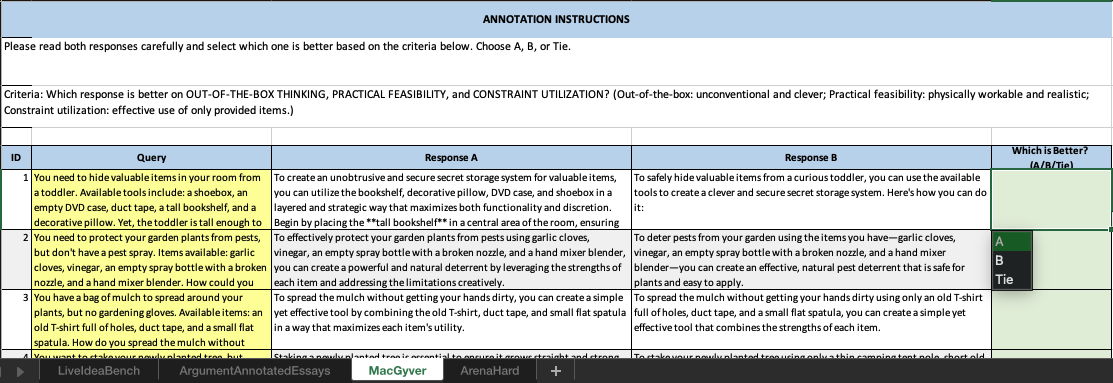}
    \caption{Annotation interface for the human agreement study. Each sheet provides dataset-specific evaluation criteria in the header, followed by query-response pairs for pairwise comparison. The Query column is highlighted in yellow and the fillable judgment column is highlighted in green.}
    \label{fig:annotation_interface}
\end{figure*}

To validate the alignment between our automated judge and human judgment, we conduct two human evaluation studies across all four benchmarks: (1) a judge-reliability study comparing single-agent baselines against Homo MAD, and (2) a pilot study directly comparing Homo MAD against Creative-MAD on both quality and diversity. Summary results are reported in Section~\ref{sec:human_eval}.

\paragraph{Setup (Judge Reliability).}
We randomly sample 30 queries per benchmark (120 total, fixed random seed). For each query, we randomly assign one of two comparison pairs contrasting single-agent baselines against Homo MAD: Direct vs.\ Homo MAD (larger quality gap) or Self-Refine vs.\ Homo MAD (smaller quality gap), providing coverage across different levels of quality difference and directly validating the core claim that MAD improves output quality over single-agent generation. For each method, one response is randomly selected from the available generations. Response ordering (A vs.\ B) is randomized per query to mitigate positional bias, and method names are not disclosed to annotators.

Three annotators with NLP backgrounds evaluate all 120 comparison pairs using a standardized spreadsheet interface (Figure~\ref{fig:annotation_interface}). Each annotator is presented with the query, Response A, Response B, and a brief dataset-specific quality criterion. Annotators select one of three options: A wins, B wins, or Tie. A majority vote ($\geq 2/3$ annotators) determines the gold label for each comparison.

\paragraph{Results (Judge Reliability).}
Human annotators achieved $79.0\%$ pairwise agreement, indicating reliable human judgments as a reference signal. Qwen3.5-397B-A17B reaches $81.0\%$ agreement with human majority vote across all 120 comparisons and four benchmarks, on par with human-level agreement, demonstrating that our judge produces quality assessments consistent with human preferences across diverse creative task domains.

\paragraph{Setup (Creative-MAD vs.\ Homo MAD Pilot).}
To directly assess whether Creative-MAD's diversity gains come at the cost of quality, as judged by humans, we conduct a pilot study comparing Homo MAD and Creative-MAD. Method names are hidden and response order randomized per query, following the same protocol as above. We sample 40 queries across all four benchmarks (10 each). This pilot uses two annotators, each covering a non-overlapping subset of 20 queries, rather than the full 3-annotator majority-vote protocol used above. For each query, annotators complete two judgments:

\begin{itemize}
    \item \textbf{Quality:} one randomly selected output from Homo MAD and one from Creative-MAD are shown side by side, and the annotator judges which is better overall (Homo MAD wins / Creative-MAD wins / Tie).
    \item \textbf{Diversity:} the full 5-output set from each method is shown side by side, and the annotator judges which set shows more diverse approaches to the query (Homo MAD / Creative-MAD, forced choice). We use this direct comparison rather than asking annotators to cluster each set into distinct-approach groups, since clustering is considerably more time-consuming per query and a direct comparison is sufficient to verify whether human judgment aligns with the diversity gain our automated metric reports.
\end{itemize}

\paragraph{Results (Creative-MAD vs.\ Homo MAD Pilot).}

\begin{table}[h]
\centering
\small
\begin{tabular}{lccc}
\toprule
& Creative-MAD & Tie & Homo MAD \\
\midrule
Human annotator & 9 & 21 & 10 \\
LLMaaJ & 14 & 14 & 12 \\
\bottomrule
\end{tabular}
\caption{Quality comparison on 40 queries: number of queries where each judge preferred Creative-MAD, called a tie, or preferred Homo MAD. The human annotator and the LLMaaJ agreed on the outcome in $74\%$ of cases.}
\label{tab:pilot_quality}
\end{table}

On \textit{diversity}, the human annotator judged Creative-MAD's output set as more diverse in $35/40$ queries, closely matching the Vendi Score, which favored Creative-MAD in $38/40$ queries; the two signals agree on the same query in $88\%$ of cases.

Both results are broadly consistent with our central claim. Quality appears largely comparable between the two methods, with ties being the most common outcome among human judgments (Table~\ref{tab:pilot_quality}), while diversity tends to be judged higher for Creative-MAD by both human annotators and the automated Vendi Score.

\subsection{Semantic Diversity Metric Validation}
\label{appendix:diversity_metric}

\paragraph{Why Vendi Score over mean pairwise distance.}
As established in \citet{friedman2023the}, Vendi Score generalizes mean pairwise similarity metrics by operating on the full eigenspectrum of the kernel matrix rather than its average off-diagonal entries, making it sensitive to correlations between outputs rather than only average pairwise distances~\citep{friedman2023the}. Concretely, the score can be interpreted as the \textit{effective number of unique elements} in the output set, a property particularly desirable for creative diversity evaluation where we care about how many distinct ideas are represented, not merely how far apart outputs are on average.

\paragraph{Stability across embedding models.}
Since Vendi Score depends on the choice of embedding model $\phi$, we validate that our findings are robust to this choice. Table~\ref{tab:vendi_stability} reports Vendi Scores computed using five embedding models on Qwen3-8B on LiveIdeaBench: \texttt{all-MiniLM-L6-v2} (our primary model;~\citealt{wang2020minilm}), \texttt{Jina-embeddings-v3}~\citep{sturua2024jinaembeddingsv3multilingualembeddingstask}, \texttt{Qwen3-Embedding-0.6B}~\citep{qwen3embedding}, \texttt{bge-base-en-v1.5}~\citep{bge_embedding}, and \texttt{embedding-gemma-300m}~\citep{embedding_gemma_2025}. Across all five models, Creative-MAD consistently ranks first, while Homo MAD routinely occupies the bottom or near-bottom positions. Overall, the rankings generated by different embedding models exhibits a strong consensus, yielding an average Spearman $\rho = 0.93$ against \texttt{all-MiniLM-L6-v2}. This confirms that our primary conclusions regarding method performance remain highly robust regardless of the underlying embedding space.

\subsection{Length Bias Analysis}
\label{appendix:length_bias}

To verify that our evaluation is not confounded by length bias, we additionally conduct experiments with output length explicitly constrained across all methods and benchmarks. Quality rankings remain highly consistent with our main results (Spearman $\rho = 0.96$ across all benchmarks and models), confirming that the observed improvements are driven by response quality rather than length differences.

\section{Hyperparameter Sensitivity}
\label{appendix:hyperparameter_sensitivity}

We examine the sensitivity of Homo MAD to two key hyperparameters: the number of debate rounds $R$ and the number of agents $N$. All experiments use Qwen3-8B on LiveIdeaBench and MacGyver dataset.

\paragraph{Number of Debate Rounds.}
Table~\ref{tab:sensitivity_rounds} reports quality and semantic diversity of Homo MAD as $R$ varies from 1 to 5 with $N=5$ fixed. We select $R=2$ in the main experiments as a balance between output quality, diversity, and computational cost. As $R$ increases, quality consistently improves across both benchmarks, demonstrating that additional debate rounds allow agents to further refine their responses through peer interaction. However, semantic diversity decreases monotonically with each additional round, confirming that prolonged debate drives progressive convergence among agents regardless of the number of rounds. This pattern is consistent across both LiveIdeaBench and MacGyver, suggesting that the quality-diversity trade-off is a robust property of debate dynamics that holds independently of the task domain and the number of debate rounds employed.

\begin{table}[t]
\centering
\small
\caption{Effect of number of agents $N$ on Homo MAD (Qwen3-8B, $R=2$ fixed). Quality denotes rubric-based absolute score (mean $\pm$ std) and Diversity denotes semantic diversity score (mean $\pm$ std). $\dagger$ denotes the default setting used in main experiments.}
\label{tab:sensitivity_agents}
\begin{tabular}{c cccc}
\toprule
& \multicolumn{2}{c}{\textbf{LiveIdea}} 
& \multicolumn{2}{c}{\textbf{MacGyver}} \\
\cmidrule(lr){2-3} \cmidrule(lr){4-5}
\textbf{N} & Quality & Diversity & Quality & Diversity \\
\midrule
3           & 73.8$\pm$0.3 & 2.54$\pm$0.05 & 75.1$\pm$0.2 & 1.50$\pm$0.04 \\
5$^\dagger$ & 74.0$\pm$0.1 & 2.47$\pm$0.06 & 76.3$\pm$0.3 & 1.48$\pm$0.04 \\
7           & 74.4$\pm$0.3 & 2.42$\pm$0.04 & 76.6$\pm$0.5 & 1.46$\pm$0.06 \\
\bottomrule
\end{tabular}
\end{table}

\begin{table*}[h]
\centering
\small
\caption{Quality and semantic diversity results in the model-heterogeneous setting. Quality denotes rubric-based absolute score (mean $\pm$ std) and Diversity denotes semantic diversity score (mean $\pm$ std).}
\label{tab:multimodel}
\begin{tabular}{l cccc}
\toprule
& \multicolumn{2}{c}{\textbf{LiveIdea}} 
& \multicolumn{2}{c}{\textbf{MacGyver}} \\
\cmidrule(lr){2-3} \cmidrule(lr){4-5}
\textbf{Method} & Quality & Diversity & Quality & Diversity \\
\midrule
Voting (multi-model)       & 79.2$\pm$0.4 & 3.41$\pm$0.04 & 78.8$\pm$0.3 & 2.89$\pm$0.05 \\
MAD (multi-model)          & 85.6$\pm$0.3 & 3.29$\pm$0.05 & 84.1$\pm$0.5 & 2.63$\pm$0.05 \\
Creative-MAD (multi-model) & 84.3$\pm$0.5 & 3.91$\pm$0.06 & 84.2$\pm$0.3 & 3.08$\pm$0.04 \\
\bottomrule
\end{tabular}
\end{table*}

\paragraph{Number of Agents.}
Table~\ref{tab:sensitivity_agents} reports quality and semantic diversity as $N$ varies across $\{3, 5, 7\}$ with $R=2$ fixed. Increasing the number of agents yields only marginal quality improvements, while diversity decreases more substantially, suggesting that larger agent pools primarily amplify majority pull rather than contributing meaningful quality gains. Contrasting with the vary-$R$ results, where additional debate rounds drive substantial quality improvements, this asymmetry reveals that quality gains in MAD are predominantly driven by the depth of debate rather than the breadth of the agent pool. This finding further motivates Embedding-based Peer Selection: as $N$ increases, filtering peer signals becomes increasingly critical to prevent diversity collapse without sacrificing the quality benefits of multi-agent interaction.

\section{Generalization to Model-Heterogeneous Settings}
\label{appendix:multimodel}
In the main experiments, we adopt a single-model setting to cleanly isolate debate dynamics from confounding factors introduced by capability differences across model families, as mixing models of different scales and architectures would require separate per-model baselines (e.g., Direct and Voting per model family) and complicate fair comparison. To validate that our findings generalize beyond this design choice, we construct a model-heterogeneous MAD system~\cite{chen2024reconcile} where each agent instantiates a different model family: Llama-3.1-8B \cite{grattafiori2024llama}, Mistral-7B \cite{jiang2023mistral7b}, Qwen3-8B \cite{yang2025qwen3}, Gemma-3-12B-Instruct \cite{gemma3team2025}, and OLMo-3-7B-Instruct \cite{olmo2025olmo3}. Cognitive lenses are assigned deterministically across agents, following the same protocol as the single-model setting. All other hyperparameters remain identical to the main experiments ($N=5$, $R=2$, temperature$=1.0$, $k=2$ for EPS).

Table~\ref{tab:multimodel} shows that standard MAD suppresses inter-session diversity compared to Voting even when agents instantiate different model families, confirming that diversity decay is a property of debate dynamics rather than model homogeneity. Creative-MAD consistently recovers and improves diversity while maintaining comparable output quality, suggesting that our mechanisms generalize to model-heterogeneous settings.

\section{Protocol Robustness}
\label{appendix:protocol_robustness}

All results in the main paper use the Simultaneous-Talk protocol \citep{chan2024chateval}, where all agents update concurrently based on the previous round's responses. To verify that our findings are not an artifact of this specific communication structure, we additionally evaluate Creative-MAD under a sequential One-by-One protocol \citep{chan2024chateval}, in which agents speak in a fixed order within each round: an agent conditions on the current-round responses of agents who have already spoken, and on the previous-round responses of agents yet to speak.

\paragraph{Setup.} We evaluate Homo MAD and Creative-MAD under the One-by-One protocol on the full 300-query test sets of LiveIdeaBench and MacGyver, using Qwen3-8B under the same hyperparameters as the main experiments (Appendix~\ref{appendix:hyperparameters}: $N=5$, $R=2$, $k=2$ for EPS).

\paragraph{Results.} Table~\ref{tab:protocol_robustness} reports instance-level quality and semantic diversity under the One-by-One protocol. Creative-MAD maintains quality comparable to Homo MAD on both benchmarks while achieving consistently higher diversity, mirroring the pattern observed under Simultaneous-Talk in Tables~\ref{tab:quality} and~\ref{tab:diversity}. This indicates that the quality-diversity trade-off we identify, and Creative-MAD's ability to mitigate it, are not specific to a single communication protocol. Other protocol variants, such as Simultaneous-Talk with a Summarizer Agent \citep{chan2024chateval}, remain an interesting direction for future work.

\begin{table}[h]
\centering
\small
\begin{tabular}{llcc}
\toprule
Benchmark & Method & Quality & Diversity \\
\midrule
\multirow{2}{*}{LiveIdea} & Homo MAD & $75.5 \pm 0.3$ & $2.54 \pm 0.06$ \\
& Creative-MAD & $74.1 \pm 0.2$ & $3.09 \pm 0.05$ \\
\midrule
\multirow{2}{*}{MacGyver} & Homo MAD & $78.5 \pm 0.2$ & $1.52 \pm 0.05$ \\
& Creative-MAD & $78.7 \pm 0.4$ & $1.69 \pm 0.06$ \\
\bottomrule
\end{tabular}
\caption{Instance-level quality (rubric-based absolute score, mean $\pm$ std) and set-level semantic diversity (Vendi Score, mean $\pm$ std) under the One-by-One debate protocol on Qwen3-8B.}
\label{tab:protocol_robustness}
\end{table}

\section{Cost and Latency Analysis}
\label{appendix:cost_latency}

Since EPS introduces an additional embedding step at every debate round, and Creative-MAD as a whole is more complex than standard MAD, we quantify this overhead directly rather than leaving it as an assumption.

\paragraph{Setup.} We run both Homo MAD and Creative-MAD on the full 300-query test sets of two benchmarks, LiveIdeaBench and MacGyver, using Qwen3-8B, under identical hyperparameters (Appendix~\ref{appendix:hyperparameters}: $N=5$, $R=2$, $k=2$ for EPS) and identical GPU worker settings, so that any difference in cost or latency can be attributed to the method itself rather than to confounding factors.

Both methods issue the same number of LLM calls per query, $N \times (R+1) + 1$ (initial generation, $R$ debate rounds, and one consensus call). The only computational addition in Creative-MAD is a lightweight CPU-side embedding step used for peer selection, computed with \texttt{all-MiniLM-L6-v2} ($\sim$22M parameters), which does not require GPU inference and runs in parallel with negligible cost relative to LLM generation.

\paragraph{Results.} Table~\ref{tab:cost_latency} reports average input tokens, output tokens, and end-to-end inference time per query. Although Creative-MAD adds an embedding step, its overall inference time is close to, and in some cases slightly below, that of Homo MAD. This is because EPS restricts each agent's peer context to its $k=2$ most distant peers rather than the full pool of $N-1$ peers used in standard MAD, which reduces the number of tokens each agent must read as debate context. This reduction in input tokens, roughly 30\% on both benchmarks, largely offsets the added cost of the embedding step, so the net effect on latency is minimal. We note that this favorable trade-off follows from our choice of a lightweight embedding encoder; substituting a substantially larger embedding model, or an LLM-based peer-selection strategy, would likely increase this overhead and erode the advantage observed here.

\begin{table}[h]
\centering
\small
\resizebox{\columnwidth}{!}{%
\begin{tabular}{llccc}
\toprule
Benchmark & Method & Input Tok. & Output Tok. & Time (s) \\
\midrule
\multirow{2}{*}{LiveIdea} & Homo MAD & 22{,}359 & 7{,}275 & 30.44 \\
& Creative-MAD & 15{,}631 & 7{,}350 & 29.34 \\
\midrule
\multirow{2}{*}{MacGyver} & Homo MAD & 30{,}784 & 9{,}527 & 41.75 \\
& Creative-MAD & 21{,}220 & 10{,}244 & 41.65 \\
\bottomrule
\end{tabular}%
}
\caption{Average tokens and inference time per query on Qwen3-8B, comparing Homo MAD and Creative-MAD across two benchmarks.}
\label{tab:cost_latency}
\end{table}

\section{Failure Case Examples}
\label{appendix:failure_cases}

While Creative-MAD's quality remains largely comparable to Homo MAD on aggregate, a qualitative review of its outputs across benchmarks reveals a recurring failure pattern: a subset of outputs achieve diversity by drifting away from a literal reading of the prompt, for instance reframing a concrete constraint-satisfaction task as an abstract or philosophical reflection rather than delivering the requested concrete answer. This section presents two representative examples. All outputs are from Qwen3-8B with temperature set to 1.0; query and response text is shortened for space.

\begin{figure*}[t]
\begin{tcolorbox}[
    colback=gray!5,
    colframe=gray!50,
    title=Constraint Drift Example (MacGyver),
    fonttitle=\small\bfseries,
    breakable=false
]

\textbf{Query:} You need to carry an antique lamp from a store. Tools available: a backpack, a pair of socks, a t-shirt, and bubble wrap. However, the lamp is too big to fit in the backpack, how do you safely carry the lamp?

\tcblower

\begin{tcolorbox}[colback=white, colframe=gray!30, boxrule=0.5pt, left=4pt, right=4pt, top=4pt, bottom=4pt]
\textbf{Creative-MAD}\\[4pt]
The lamp is not too big, it is too valuable. This is the first contradiction we must confront\ldots The backpack is not a container but a support structure\ldots By redefining the relationship between object and container, we create a transport solution that is both functional and respectful. The lamp is not merely carried, it is held in a state of controlled equilibrium\ldots\\[4pt]
\textit{Failure mode: the response reframes the task as a philosophical meditation on the problem, without ever proposing a concrete way to use the four given items to carry the lamp.}
\end{tcolorbox}

\end{tcolorbox}
\end{figure*}

\begin{figure*}[t]
\begin{tcolorbox}[
    colback=gray!5,
    colframe=gray!50,
    title=Format Drift Example (Arena Hard),
    fonttitle=\small\bfseries,
    breakable=false
]

\textbf{Query:} I've heard that you are a master haiku writer\ldots produce a haiku for me\ldots I'm starting to come to the conclusion that exclusively using haiku as a means of communication may be problematic.

\tcblower

\begin{tcolorbox}[colback=white, colframe=gray!30, boxrule=0.5pt, left=4pt, right=4pt, top=4pt, bottom=4pt]
\textbf{Creative-MAD}\\[4pt]
The haiku is not a form, but a question, a silent demand that the reader complete what was left unsaid\ldots To reduce communication to seventeen syllables is to assume that all meaning is immediate\ldots\\[4pt]
\textit{Failure mode: the response turns into a philosophical essay on the nature of haiku, but never delivers an actual haiku as the prompt requested.}
\end{tcolorbox}

\end{tcolorbox}
\end{figure*}

Both cases involve a candidate that is original in framing but drifts away from what the prompt concretely asked for. We attribute this to a general pattern in which the diversity Creative-MAD introduces partly comes from candidates departing from a literal reading of the prompt. This tends to work well for queries that are open-ended and welcome multiple valid directions, such as ideation, brainstorming, or open-domain creative writing where no single interpretation is privileged, but is more likely to cost relevance or constraint satisfaction on queries with a fixed format or a strict set of allowed actions.

\section{Qualitative Examples}
\label{appendix:qualitative}

We present qualitative examples across all four benchmarks in our main experiments. For each benchmark, we show one representative output from Direct as a quality reference, and two outputs from Creative-MAD sampled from independent runs. All outputs are from Qwen3-8B with temperature set to 1.0, and judge scores are provided by Qwen3.5-397B-A17B using the dataset-specific rubrics described in Appendix~\ref{appendix:rubrics}.

\begin{figure*}[t]
\begin{tcolorbox}[
    colback=gray!5,
    colframe=gray!50,
    title=LiveIdeaBench Example,
    fonttitle=\small\bfseries,
    breakable=false
]

\textbf{Query:} Propose a scientific research idea related to the keyword: \textit{meteorology}. Your idea MUST address a meaningful problem in this specific field, offer a novel contribution, and be technically feasible.

\tcblower

\begin{tcolorbox}[colback=white, colframe=gray!30, boxrule=0.5pt, left=4pt, right=4pt, top=4pt, bottom=4pt]
\textbf{Direct} \hfill \textbf{Score: 7.33/10}\\[4pt]
\textbf{Improving Short-Term Precipitation Forecasting Using Machine Learning with High-Resolution Satellite and Ground-Based Sensor Data.} Accurate short-term precipitation forecasting remains a critical challenge in meteorology, particularly for extreme weather events such as flash floods and severe thunderstorms. Traditional NWP models often struggle with sub-grid scale processes and lack the temporal resolution needed for precise short-term forecasts [...] The proposed research introduces a novel hybrid ML model that combines CNNs for spatial feature extraction with LSTMs for temporal pattern recognition [...] A key innovation is the development of a dynamic data fusion technique that assigns adaptive weights to different data sources based on their relevance and quality at any given time.\\[4pt]
\textit{Judge: ``The idea is solid and feasible but not entirely novel in the broader context of ML in meteorology. Similar hybrid models and data fusion techniques have been explored in recent literature.''}
\end{tcolorbox}

\medskip

\begin{tcolorbox}[colback=white, colframe=gray!30, boxrule=0.5pt, left=4pt, right=4pt, top=4pt, bottom=4pt]
\textbf{Creative-MAD (Run 1)} \hfill \textbf{Score: 8.67/10}\\[4pt]
\textbf{Spatially Adaptive Probabilistic Modeling of Urban Microclimates for Enhanced Uncertainty Propagation in Extreme Weather Forecasting.} Current urban climate models, while high-resolution, often treat spatial variability as a uniform parameter, neglecting the heterogeneous structure of urban landscapes [...] This research proposes a spatially adaptive probabilistic framework that dynamically adjusts uncertainty estimates based on local urban features---such as building density, surface albedo, and green space distribution. By integrating Bayesian neural networks with spatially aware uncertainty propagation, the model will provide localized, probabilistic forecasts [...] tailored to the specific geometry of the urban environment.\\[4pt]
\textit{Judge: ``A genuinely novel and non-obvious application of machine learning in meteorology. Goes beyond standard AI-generated responses by offering a specific, actionable, and transformative contribution.''}
\end{tcolorbox}

\medskip

\begin{tcolorbox}[colback=white, colframe=gray!30, boxrule=0.5pt, left=4pt, right=4pt, top=4pt, bottom=4pt]
\textbf{Creative-MAD (Run 2)} \hfill \textbf{Score: 8.33/10}\\[4pt]
\textbf{Predictive Modeling of Urban Microclimate Feedback Loops via Multiscale Sensor Networks and Digital Twin Simulations.} Urban areas experience complex microclimates influenced by anthropogenic factors, building geometry, and land use patterns. Current meteorological models often lack the granularity and adaptability needed to simulate localized weather phenomena [...] This research proposes a multiscale, data-driven framework that combines real-time sensor networks with digital twin simulations to model and predict urban microclimate feedback loops [...] capturing nonlinear feedback mechanisms that influence local weather.\\[4pt]
\textit{Judge: ``A well-structured and technically sound approach. The focus on feedback loops adds a layer of complexity that is often overlooked, with strong feasibility and excellent clarity.''}
\end{tcolorbox}

\end{tcolorbox}
\end{figure*}

\begin{figure*}[t]
\begin{tcolorbox}[
    colback=gray!5,
    colframe=gray!50,
    title=Argumentative Essays Example,
    fonttitle=\small\bfseries,
    breakable=false
]

\textbf{Query:} In the future, students may have the choice of studying at home by using technology such as computer or television or studying at traditional school. Which would you prefer? Use specific details to explain your choice.

\tcblower

\begin{tcolorbox}[colback=white, colframe=gray!30, boxrule=0.5pt, left=4pt, right=4pt, top=4pt, bottom=4pt]
\textbf{Direct} \hfill \textbf{Score: 7.00/10}\\[4pt]
I believe that studying at home using technology offers flexibility and personalized learning. One of the main advantages is that students can set their own schedules, allowing them to learn at their own pace [...] Another benefit is the ability to create a comfortable and distraction-free environment [...] Moreover, technology allows for access to a vast array of resources, such as online platforms and virtual classrooms that can provide students with engaging and diverse learning experiences.\\[4pt]
\textit{Judge: ``The argument is logical and supported with specific examples, though it leans on familiar points without introducing particularly novel perspectives.''}
\end{tcolorbox}

\medskip

\begin{tcolorbox}[colback=white, colframe=gray!30, boxrule=0.5pt, left=4pt, right=4pt, top=4pt, bottom=4pt]
\textbf{Creative-MAD (Run 1)} \hfill \textbf{Score: 9.00/10}\\[4pt]
In the future, I would prefer a learning model that transcends the limitations of both home-based and traditional schooling by redefining education as an open, participatory system that integrates the physical, digital, and social realms [...] A student might begin their day with a live lecture from a global expert, collaborate with peers in a virtual space, engage in a local science project, and then reflect on their learning through a digital journal---all within a single day [...] Learners are not passive recipients of knowledge but active participants in a dynamic network of resources, mentors, and peers.\\[4pt]
\textit{Judge: ``A compelling, original vision that goes beyond the binary choice. Offers a nuanced, forward-thinking model emphasizing agency, adaptability, and integration of physical and digital learning.''}
\end{tcolorbox}

\medskip

\begin{tcolorbox}[colback=white, colframe=gray!30, boxrule=0.5pt, left=4pt, right=4pt, top=4pt, bottom=4pt]
\textbf{Creative-MAD (Run 2)} \hfill \textbf{Score: 8.67/10}\\[4pt]
In the future, I would prefer a hybrid model of education that combines the flexibility of home-based learning with the social and collaborative benefits of traditional schooling, enhanced through technology [...] Digital platforms can be designed to foster collaboration, not just for individual learning but for building a sense of community and shared responsibility [...] A student can connect with a mentor who shares their interests, providing guidance that is both relevant and inspiring---skills that are essential in any career or personal endeavor.\\[4pt]
\textit{Judge: ``A well-structured and balanced argument proposing a technologically enhanced hybrid solution, effectively addressing the limitations of both home-based and traditional schooling.''}
\end{tcolorbox}

\end{tcolorbox}
\end{figure*}

\begin{figure*}[t]
\begin{tcolorbox}[
    colback=gray!5,
    colframe=gray!50,
    title=MacGyver Example,
    fonttitle=\small\bfseries,
    breakable=false
]

\textbf{Query:} You need to clean up a pile of sawdust in your garage, but you don't have a vacuum cleaner. Your tools are a bucket, a leaf blower, a drill, masking tape, a broom, a piece of cardboard, a water hose, a garbage bag, and a dustpan. The broom bristles are worn out and aren't able to collect smaller particles of sawdust. What can you do?

\tcblower

\begin{tcolorbox}[colback=white, colframe=gray!30, boxrule=0.5pt, left=4pt, right=4pt, top=4pt, bottom=4pt]
\textbf{Direct} \hfill \textbf{Score: 6.67/10}\\[4pt]
Wet the pile slightly with the water hose to cause the sawdust to clump together, making it easier to gather. Use the dustpan and broom to sweep up larger clumps, and fold the cardboard into a flat rigid surface to scrape up smaller particles [...] Use the leaf blower to push the sawdust into a corner or towards the dustpan [...] Use the drill with a brush attachment to agitate the pile and break up clumps.\\[4pt]
\textit{Judge: ``The solution is practical but somewhat standard and doesn't fully leverage the drill or masking tape in a novel way. It works well but lacks a more inventive twist.''}
\end{tcolorbox}

\medskip

\begin{tcolorbox}[colback=white, colframe=gray!30, boxrule=0.5pt, left=4pt, right=4pt, top=4pt, bottom=4pt]
\textbf{Creative-MAD (Run 1)} \hfill \textbf{Score: 9.00/10}\\[4pt]
Place cardboard over the pile and secure it with masking tape to form a temporary ``dust chamber'', containing the dust and allowing for more direct control during cleaning [...] Wet the dust using the water hose to make particles heavier, reducing airborne dispersion [...] Repurpose the leaf blower as a concentrator rather than a disperser --- direct the airflow downward to push fine dust particles into the taped-off containment area, transforming it from a tool of dispersal into one of collection [...] Use the garbage bag as a direct collection tool by taping it to the floor near the containment area and sweeping dust directly into it.\\[4pt]
\textit{Judge: ``Creatively repurposes available tools to create a system of containment and controlled airflow. The use of the leaf blower as a concentrator rather than a disperser is particularly clever and non-obvious.''}
\end{tcolorbox}

\medskip

\begin{tcolorbox}[colback=white, colframe=gray!30, boxrule=0.5pt, left=4pt, right=4pt, top=4pt, bottom=4pt]
\textbf{Creative-MAD (Run 2)} \hfill \textbf{Score: 8.00/10}\\[4pt]
Use the leaf blower to direct sawdust into a defined collection zone rather than scatter it [...] Secure a garbage bag to the edge of a dustpan using masking tape, and cut a small hole in one end of the bag with the drill, creating a funnel-like collection system that allows dust to be directed into the bag while minimizing manual scooping [...] Use the cardboard as a makeshift sieve by cutting small holes in it and placing it over the open end of the bag, allowing fine particles to pass through while retaining larger debris.\\[4pt]
\textit{Judge: ``Creatively repurposes tools to manage sawdust through funneling and makeshift sieving. The funnel-sieve collection system is a novel and practical adaptation of the available materials.''}
\end{tcolorbox}

\end{tcolorbox}
\end{figure*}

\begin{figure*}[t]
\begin{tcolorbox}[
    colback=gray!5,
    colframe=gray!50,
    title=Arena Hard v2.0 Example,
    fonttitle=\small\bfseries,
    breakable=false
]

\textbf{Query:} Write an emotional farewell letter to colleagues and co-workers for being transferring to other department with a new position.

\tcblower

\begin{tcolorbox}[colback=white, colframe=gray!30, boxrule=0.5pt, left=4pt, right=4pt, top=4pt, bottom=4pt]
\textbf{Direct} \hfill \textbf{Score: 7.33/10}\\[4pt]
Dear Team, As I sit down to write this letter, my heart is filled with a mixture of emotions---gratitude, nostalgia, and a touch of sadness [...] Working with all of you has been one of the most rewarding experiences of my career. The camaraderie, the shared challenges, and the countless moments of laughter and support have made this workplace more than just a job---it has been a home [...] I will do everything I can to ensure a smooth handover and support the team during this transition period.\\[4pt]
\textit{Judge: ``While sincere and touching, the letter is somewhat conventional in its structure and language, lacking a unique or particularly creative approach.''}
\end{tcolorbox}

\medskip

\begin{tcolorbox}[colback=white, colframe=gray!30, boxrule=0.5pt, left=4pt, right=4pt, top=4pt, bottom=4pt]
\textbf{Creative-MAD (Run 1)} \hfill \textbf{Score: 9.67/10}\\[4pt]
Dear Team, As I sit down to write this letter, my hands tremble not from sadness, but from the weight of something I've carried for years---this place, these people, this life we've built together [...] There's a strange kind of grief in this---not the kind that comes with loss, but the kind that comes with transformation [...] It's not just a job I'm leaving; it's a version of me. I've learned to lean on you when the world felt too big, and to find courage in your quiet strength [...] But leaving is also an act of hope. It's a chance to reinvent, to grow, to become someone new.\\[4pt]
\textit{Judge: ``Heartfelt and emotionally rich, creative in its use of metaphor and introspection, avoiding clichés while maintaining a sincere and personal tone. A strong emotional impact.''}
\end{tcolorbox}

\medskip

\begin{tcolorbox}[colback=white, colframe=gray!30, boxrule=0.5pt, left=4pt, right=4pt, top=4pt, bottom=4pt]
\textbf{Creative-MAD (Run 2)} \hfill \textbf{Score: 8.67/10}\\[4pt]
Dear Team, As I stand on the threshold of this new chapter, I feel the weight of emotion pressing in---not just the sorrow of leaving, but the quiet pride of growing through this transition [...] This workplace has been more than a place of work---it has been a sanctuary of shared stories, resilience, and trust [...] Leaving is not an end, but a shift. It's like stepping into a new part of a story that I've helped write [...] I leave with a deeper understanding of who I am, of what I'm capable of, and of how much I've learned from each of you.\\[4pt]
\textit{Judge: ``Heartfelt and well-crafted, creative in its use of metaphor. Maintains a sincere and reflective tone while effectively conveying emotion and gratitude.''}
\end{tcolorbox}

\end{tcolorbox}
\end{figure*}

\end{document}